\documentclass[USenglish,onecolumn]{article}

\usepackage[utf8]{inputenc}
\usepackage{PRIMEarxiv}
 \makeatletter
\newcommand\schapter[1]{%
    \if@openright\cleardoublepage\else\clearpage\fi
    \phantomsection 
    \@schapter{#1}%
    \addcontentsline{toc}{chapter}{#1} 
}
\providecommand*\toclevel@schapter{0} 
\providecommand*\Hy@toclevel@schapter{0} 
\makeatother

\usepackage[utf8]{inputenc} 
\usepackage[T1]{fontenc}    
\usepackage{hyperref}       
\usepackage{url}            
\usepackage{booktabs}       
\usepackage{amsfonts}       
\usepackage{nicefrac}       
\usepackage{microtype}      
\usepackage{xcolor}         
\usepackage{amsmath}
\usepackage{amsthm}
\usepackage{amssymb}
\usepackage{xparse}
\usepackage{graphicx}
\usepackage{cancel}
\usepackage{comment}
\usepackage{pifont}
\usepackage{mathtools}
\usepackage[noabbrev,nameinlink]{cleveref}
\usepackage{enumitem}
\setlist[enumerate]{leftmargin=2em}
\setlist[itemize]{leftmargin=2em}
\setlist[description]{leftmargin=*}
\usepackage{tcolorbox}

\RequirePackage{fix-cm}  
\usepackage{anyfontsize}  

\usepackage{thm-restate} 

\hypersetup{
    breaklinks=true,
    pdfencoding=unicode,
    bookmarksnumbered,
    pdfborder={0 0 0},
}%

\newcommand{\cmark}{\ding{51}}%
\newcommand{\xmark}{\ding{55}}%
\definecolor{lightblue}{rgb}{0.6 , 0.7, 0.9}
\definecolor{lightpurple}{rgb}{0.65 , 0.6, 0.9}
\definecolor{lightcyan}{rgb}{0.5 , 0.85, 0.8}
\definecolor{darkpurple}{rgb}{0.35 , 0.15, 0.6}

\newtheorem{theorem}{Theorem}[section]
\newtheorem{definition}[theorem]{Definition}

\newtheorem{lemma}[theorem]{Lemma}
\newtheorem{corollary}[theorem]{Corollary}

\newtheorem{proposition}[theorem]{Proposition}

\NewDocumentCommand{\optionalbracket}{m m}{%
    \NewDocumentCommand{#1}{g}{%
        \IfNoValueTF{##1}{%
            #2%
        }{%
            #2(##1)%
        }%
    }%
}

\NewDocumentCommand{\oneortwo}{m m m}{%
    \NewDocumentCommand{#1}{g}{%
        \IfNoValueTF{##2}{%
            #3(##1)%
        }{%
            #3(##1,##2)%
        }%
    }%
}

\newcommand{\distrsym}{\bigtriangleup}
\optionalbracket{\distributions}{\distrsym}
\optionalbracket{\distributionsO}{\distrsym_0}
\optionalbracket{\distributionsF}{\distrsym^\otimes}
\optionalbracket{\distributionsFO}{\distrsym^*_0}

\newcommand{\factors}{\Omega}
\newcommand{\indep}{\mathrel{\perp\!\!\!\perp}}
\newcommand{\orth}{\mathrel{\perp}}
\newcommand{\orthF}{\mathrel{\perp}^\FS}
\newcommand{\sect}{\cap}
\newcommand{\sectbig}{\bigcap}
\newcommand{\union}{\cup}

\newcommand{\indepstar}{\indep^\otimes}
\newcommand{\timesbig}{\mathop{\text{\Large$\times$}}}%

\newcommand{\comp}[1]{{\overline{#1}}}

\newcommand{\impl}{\Rightarrow}

\newcommand{\nothing}{\varnothing}
\newcommand{\given}{\mid}
\optionalbracket{\interiorset}{S'}
\optionalbracket{\dirac}{\delta_x}

\newcommand{\functionOf}{\triangleleft}
\newcommand{\functionTo}{\triangleright}
\newcommand{\without}{\setminus}

\optionalbracket{\PA}{\text{Pa}}
\newcommand{\pa}{\text{pa}}
\optionalbracket{\range}{\text{range}}

\newcommand{\FS}{\Omega} 
\newcommand{\FSModel}{\mathcal{M}} 
\newcommand{\Omg}{\Omega} 
\newcommand{\omg}{\omega} 
\newcommand{\Obs}{{\textnormal{Obs}}} 
\newcommand{\A}{\text{an}} 
\newcommand{\Val}{\textnormal{Val}} 
\newcommand{\Bayes}{\mathcal{B}}
\newcommand{\Xb}{X}
\newcommand{\pav}{{\pa(v)}}%
\newcommand{\xpav}{x_\pav}%
\newcommand{\Xpav}{X_\pav}%
\newcommand{\vxpav}{{(v,\xpav)}}%
\newcommand{\Pomg}{P^\Omg}%
\newcommand{\iinI}{{i\in I}}%
\newcommand{\vinV}{{v\in V}}%
\newcommand{\merge}{\cdot}%
\newcommand{\timesOmg}{\bigtimes_\iinI\Omg_i}%
\newcommand{\FSG}{\FS}%
\newcommand{\x}{\alpha}%
\newcommand{\y}{\beta}%
\newcommand{\FSMG}{\FSModel^G}%

\renewcommand{\epsilon}{\varepsilon}%

\makeatletter
\renewcommand{\paragraph}{\@ifstar\@paragraphbold\@paragraphbold}
\newcommand{\@paragraphbold}[1]{\vspace{4mm}\noindent\textbf{{#1}.}}
\makeatother

\optionalbracket{\cvar}{C}
\optionalbracket{\base}{\text{Base}}
\optionalbracket{\history}{\mathcal{H}}
\optionalbracket{\relevant}{\mathcal{R}}
\optionalbracket{\supp}{\mathrm{supp}}
\optionalbracket{\proj}{\pi}
\optionalbracket{\marg}{\textnormal{marg}}

\newcommand{\compose}{\circ}

\newcommand{\before}{\leq}
\newcommand{\strictlybefore}{<}
\newcommand{\strictlybeforeF}{\strictlybefore^\FS}
\newcommand{\beforeF}{\before^\FS}

\DeclareMathOperator{\historyC}{Cohistory}

\newtcolorbox{definitiontext}[1][]{
  colback=white,
  colframe=darkgray,
  boxrule=0.5mm,
  #1
}

\author{
  Scott Garrabrant\thanks{Equal contribution}\\
  Machine Intelligence Research Institute \\
  \texttt{scott.garrabrant@gmail.com}\\
   \And
  Matthias G. Mayer\footnotemark[1] \\
  Machine Intelligence Research Institute, \\
  Independent \\
  \texttt{matthias.georg.mayer@gmail.com}\\
   \And
  Magdalena Wache\footnotemark[1] \\
  Machine Intelligence Research Institute, \\
  Principles of Intelligent Behavior in \\
  Biological and Social Systems\\
  \texttt{magdalena-wache@mailbox.org}\\
  \And
  Leon Lang \\
  University of Amsterdam\\
  \texttt{l.lang@uva.nl}\\
  \qquad\qquad\qquad\qquad\qquad\qquad\qquad\\
   \And
  Sam Eisenstat \\
  Machine Intelligence Research Institute\\
  \texttt{sam@intelligence.org}\\
   \And
  Holger Dell \\
  Goethe University Frankfurt, Germany, \\
  IT University of Copenhagen, Denmark \\
  \texttt{hold@itu.dk} \\
}

\title{Factored space models: Towards causality between levels of abstraction}
\begin{document}

\maketitle
\begin{abstract}
{Causality plays an important role in understanding intelligent behavior, and there is a wealth of literature on mathematical models for causality, most of which is focused on causal graphs. Causal graphs are a powerful tool for a wide range of applications, in particular when the relevant variables are known and at the same level of abstraction.
However, the given variables can also be unstructured data, like pixels of an image. Meanwhile, the causal variables, such as the positions of objects in the image, can be arbitrary deterministic functions of the given variables. Moreover, the causal variables may form a hierarchy of abstractions, in which the macro-level variables are deterministic functions of the micro-level variables.
Causal graphs are limited when it comes to modeling this kind of situation.
In the presence of deterministic relationships there is generally
no causal graph that satisfies both the Markov condition and the faithfulness condition.
We introduce \emph{factored space models} as an alternative to causal graphs which naturally represent both probabilistic and deterministic relationships at all levels of abstraction.
Moreover, we introduce \emph{structural independence} and establish that it is equivalent to statistical independence in every distribution that factorizes over the factored space. This theorem generalizes the classical soundness and completeness theorem for d-separation.
}
\end{abstract}

\allowdisplaybreaks
\section{Introduction}

Learning causal relationships plays a central role in human intelligence \cite{sloman2015causality} \cite{gopnik2012reconstructing}, and learning a causal model is necessary for any system that can generalize out of distribution \cite{richens2024robust}. 
Therefore, a rigorous mathematical theory of causality is key to understanding intelligent behavior.

There is a large body of literature about the mathematical foundations of causality \cite{pearl2009causality}, a highly successful area of research, most of which is based on causal graphs.
An important feature of causal graphs is that they can encode statistical independence in their structure via \emph{d-separation}~\cite{koller2009probabilistic}. It is commonly assumed that a graph $G$ is a useful model for a probability distribution $P$ if it is a \emph{perfect map}  of $P$, that is, if d-separation in $G$ implies statistical independence in $P$ (Markov condition), and the converse (faithfulness condition) \cite{spirtes2001causation}.

\begin{figure}[ht]
\centering
\includegraphics[width=0.75\textwidth]{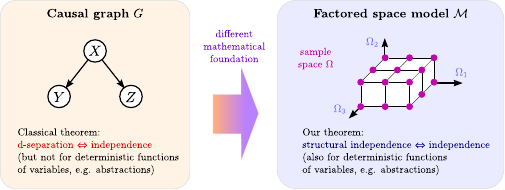}
\caption{Summary of our results. In a causal graph, the standard criterion for independence is d-separation. However, d-separation is undefined for variables that are deterministic functions of other variables, such as $A\coloneq X+Y+Z$. To address this issue, we introduce factored space models (FSMs) as an alternative to causal graphs. FSMs are based on expressing the sample space as a Cartesian product, and can be visualized as a hyper-rectangle. FSMs allow us to define \emph{structural independence} as an independence criterion which is also defined for deterministic functions of variables. Further, we establish a theorem that generalizes the soundness and completeness of d-separation \cite{koller2009probabilistic} to structural independence.\vspace{-4mm}
}
\label{fig:fig1}
\end{figure}

However, as we show in \Cref{sec:background}, when variables are deterministically related, there is generally no perfect map with these variables.
This is an important limitation because deterministic relationships appear in many applications. For example, mathematical relationships, equations in physics, and computer programs all contain deterministic relationships between variables. In particular, deterministic relationships are inherent when the causal variables are unknown, such as in \emph{causal representation learning}~\cite{scholkopf2021toward}. For example, if the given variables are the pixels of an image, then the causal variables such as the positions of objects in the image can be arbitrary deterministic functions of the given variables.
Moreover, when we model a system at \emph{different levels of abstraction}, there are deterministic relationships between the detailed, micro-level variables and the abstract, macro-level variables.
For example, the temperature of gas in a container is a deterministic function of the kinetic energy of all the individual gas particles in the container. 

We introduce a framework in which works naturally with deterministic relationships, and treats functions of variables on equal footing.
Our main contributions are the following:
\begin{enumerate}[itemsep=0pt]
    \item \textbf{Factored space models} (FSMs), an alternative to causal graphs. FSMs naturally represent all random variables, including deterministic functions of other variables, rather than distinguishing those variables that appear in a particular graph. FSMs are a reframing of \emph{finite factored set models} \cite{garrabrant2021temporal}, and they are based on representing the sample space as a Cartesian product, as depicted in \Cref{fig:fig1}.
    \item \textbf{Structural independence}, a criterion that characterizes statistical independence. 
    \item \textbf{Generalizing the soundness and completeness of d-separation.}     We establish \Cref{thrm:soundcomplete}, which  generalizes the soundness and completeness theorem for d-separation \cite{koller2009probabilistic} to variables that can be deterministically related. 
\end{enumerate}

\paragraph*{Overview} The paper proceeds as follows. In section \Cref{sec:background}, we show the limitations of causal graphs when it comes to modeling deterministic relationships. 
In \Cref{sec:fsm}, we introduce the FSM framework, prove that structural independence forms a semigraphoid, and show how to construct an FSM from a causal graph. In \Cref{sec:soundcomplete}, we prove the soundness and completeness of structural independence. 
Finally, in \Cref{sec:discussion}, we discuss the implications of our results and provide directions for future work.

\section{Related work}
\label{sec:relatedwork}
In this section we review the related work on causal abstractions, and on representing deterministic relationships between variables. We also compare FSMs with causal graphs.

\subsection{Causal abstractions}
The \emph{causal abstractions} framework \cite{beckers2019abstracting} \cite{beckers2020approximate} \cite{rubenstein2017causal} \cite{zennaro2022abstraction} shares our broad aim, that of understanding relationships between variables not well modeled by causal graphs, and in particular when different variables reside at different levels of abstraction.
In the causal abstractions framework, different levels of abstractions are conceived as different causal models.
\emph{Transformations} between such models are then used to relate them.

In our work, we seek to provide a language that can simultaneously do the work of causal models and d-separation, and also express the sort of relations between variables that can occur when we speak of microscale and macroscale variables together. For example, we may describe a gas both on a very detailed level in terms of the kinetic energy of each gas particle, and on a more abstract level in terms of the overall temperature of the gas. 
In particular, a representation for different levels of abstractions should include deterministic relations---we may want to model a macroscale variable as being determined by the microscale variables.
One way to regard the division that exists in causal abstractions framework between the causal models on the one hand and the transformations on the other is that it stems from an inability to place variables that stand in such deterministic relations together in one causal graph.
This is because no such graph can be a perfect map once we account for these deterministic relations, as we show in \Cref{sec:background}.
We do not develop this here, but we are interested in the application of such ideas to neural network interpretability. The causal abstraction framework has been used in this way \cite{geiger2021causal} \cite{geiger2023causal}.

\subsection{Representing deterministic relationships}
In a causal graph, the nodes are random variables, and the independences between those variables are represented by d-separation of the nodes. To represent independences between deterministic functions of variables, we need to represent these functions of variables as their own nodes. However, as illustrated in \Cref{sec:background}, when we allow variables that are functions of other variables, d-separation becomes incomplete. There have been various attempts to model causality with variables that are deterministically related. There are approaches like Causal Constraints models \cite{blom2020beyond} and an axiomatization of causal models with constraints \cite{beckers2023causal}, which extend causal models to include variables that are deterministically related by a constraint.

However, these approaches only consider an additional set of functional constraints on a fixed set of given variables rather than modeling any deterministic function of variables as its own variable. Moreover, both of these approaches do not provide a criterion for independence in the same way that d-separation provides an independence criterion nodes in a causal graph. 
In contrast, structural independence on an FSM, our analog of d-separation, is a criterion for independence that correctly represents independence even when variables are deterministically related.

There is work which does consider the independence between deterministically related variables \cite{lee2017causal}, which provides a causal discovery method based on independence of functions of variables.
This type of independence between functions of variables is what we represent in FSMs.
However, the existing work only applies to the special case of 2 binary variables, and does not provide a general model for how to represent such independences.
In contrast, FSMs can model situations with an arbitrary number of discrete variables.
The closest to an independence criterion for variables which can have deterministic relationships is D-separation \cite{geiger1990identifying} (note the capital `D' to distinguish from d-separation). But just like for d-separation there are simple counterexamples for variables that are independent but not D-separated, as shown in \Cref{sec:background}.

\subsection{Causal graphs and FSMs}
Like a causal graph, an FSM models a set of probability distributions on some variables, and implies certain conditional independence relations between them.
Other research has also examined ways of representing probability distributions.
In particular, the factors of an FSM can be regarded as factors of a \emph{factor graph} \cite{witten2002data}.
FSMs generalize causal graphs; we can construct an FSM from a causal graph in order to represent the same independence properties, as we demonstrate in \Cref{sec:dagconstruction}.
Here, the factors are independent, which is in line with the spirit of causal modeling, as expressed by the principle of independent mechanisms \cite{peters2017elements}.

\subsubsection{Soundness and completeness of d-separation}
The theory of causal graphs provides us with the d-separation criterion \cite{koller2009probabilistic}, which characterizes which conditional independence relations are implied by a causal graph.
Given three sets of variables $X, Y, Z$ corresponding to sets of nodes of a causal graph, $X$ and $Y$ are conditionally independent given $Z$ in all probability distributions that factorize according to the graph if and only if $X$ and $Y$ are d-separated by $Z$ in the graph.

We prove a more general version of this theorem, in which d-separation in a graph $G$ is replaced by \emph{structural independence} in a \emph{factored space}. Our theorem is more general because in the classic theorem, $X$, $Y$, and $Z$ are nodes in a graph, while in our version they can be arbitrary variables, including deterministic functions of variables such as $X+Y$. In particular, when an FSM is constructed from a causal graph, the classic theorem follows from our theorem.

\section{Causal graphs cannot capture deterministic relationships}
\label{sec:background}
Methods based on causal graphs are highly successful in modeling independence, intervention and counterfactuals. However, as we illustrate in this section,
when there is a deterministic relationship between random variables, there is generally no \emph{perfect map} of the probability distribution. That is, there is no graph $G$ such that d-separation in $G$ is equivalent to statistical independence.

Let $\vec{X}\coloneq(X_1, \dots, X_n)$ be a vector-valued random variable, and let $Y\coloneq \frac{1}{n}\sum_{i=1}^n X_i$ be the mean of $\vec{X}$. One can think of $Y$ as an abstraction of $X$. Further, let $Z$ be a variable that is causally downstream of $Y$, for example $Z\coloneq Y+N$ wherein $N$ is independent noise. Then $\vec{X}$ and $Z$ are independent given $Y$ (denoted $\vec{X}\indep Z\given Y$), since when $Y$ is known, $\vec{X}$ provides no further information about $Z$. Moreover, $Y \indep Z\given\vec{X}$ holds, since knowing $\vec{X}$ fully determines $Y$, so $Z$ provides no further information about $Y$. Furthermore, we have $Z\not\indep (Y, \vec{X})$.

\begin{figure}[ht]
\centering
\includegraphics[width=0.2\textwidth]{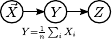}
\caption{This graph violates the faithfulness condition: $Y$ and $Z$ are not d-separated by $\vec{X}$, despite $Y\indep Z\given \vec{X}$.}
\label{fig:not_perfect}
\end{figure}
These three independence statements together contradict the \emph{intersection axiom} \cite{pearl2022graphoids} of d-separation ($\text{d-sep}(X, Y\given Z) \land \text{d-sep}(X,Z\given Y)\impl \text{d-sep}(X, (Y,Z))$). Therefore, there is no graph in which the d-separations are equivalent to the statistical independencies.
In particular, the graph in \Cref{fig:not_perfect} is not a perfect map, because it does not reflect that $Y \indep Z\given\vec{X}$ holds.

Even \emph{D-separation} \cite{geiger1990identifying}, which extends d-separation to include nodes that are deterministic functions of their parents, does not cover this case.

\section{Factored space models}
\label{sec:fsm}
In this section, we formally introduce the factored space model (FSM) framework. We define the building blocks of the framework, including \emph{derived variables}, \textit{history}, \textit{structural independence} and \textit{structural time}.

\paragraph*{Derived variables}
Random variables are a central object of study in statistics and causality. Formally, a \emph{random variable} on a sample space $\Omega$ is a measurable function $X\colon\Omg \to \Val(X)$ from $\Omg$ to a value space $\Val(X)$. When we speak of variables in this paper, we always mean discrete random variables. Any subset of $\Omg$ is called \emph{event}.

Our contribution is a better way to represent variables that are deterministically related. That is, some variables are a deterministic function of other variables. We define a variable being a deterministic function of another variable  as follows:

\begin{definition}[Derived Variable]\label{def:derived}
  Let $X \colon  \Omg \to \Val(X)$ and $Y \colon  \Omg \to \Val(Y)$ be two random variables, and let $C\subseteq\Omg$ be an event. Then, we say that \emph{$Y$ is derived from $X$ on $C$}, or that it is a \emph{deterministic function of} $X$ on $C$ if there is a function $f\colon  \Val(X) \to \Val(Y)$, such that for all $\omg \in C$, it holds that $Y(\omg) = f(X(\omg))$. We also write this as $X \functionTo_C Y$. We write $X \functionTo Y$ as a shorthand for $X\functionTo_\Omg Y$, and say that $Y$ is derived from $X$ or a deterministic function of $X$.

\end{definition}

Given multiple variables $X_1, \dots, X_n$, we define the variable $(X_1,\dots, X_n)\colon  \Omg \to \Val(X_1)\times \dots\times\Val(X_n)$ by $\omg \mapsto (X_1(\omg), \dots, X_n(\omg))$. Then, if $Y\functionOf (X_1, \dots, X_n)$, we say that $Y$ is a deterministic function of the variables $X_1, \dots, X_n$.

\paragraph*{Operations on indexed families}
\label{par:indexedfamilies}
As we make heavy use of indexed families, we introduce some operations on indexed families and probability distributions over sets of indexed families. Formally, an indexed family $a$ with the index set $I$, written as $a=(a_i)_\iinI $ is a function $f_a\colon I\to T$ from an index set $I$ to a target set $T$ of possible values. However, rather than being viewed as a function, $a$ is treated as a collection of indexed elements $a_i \coloneqq f_a(i)$. For example, when we write $x\in a$, that means there is an $i\in I$ with $x=a_i$. We define the following operations.
\begin{enumerate}
    \item \textbf{Projection.}
Let $a=(a_i)_\iinI $ be an indexed family. Let $j\in I$, and $J\subseteq I$.  Then the \textit{projection} of $a$ to $j$ is defined as $\proj_j(a)\coloneqq a_j$. The projection of $a$ to $J$ is defined as $\proj_J(a)\coloneqq (a_j)_{j\in J}$.
Further, let $A$ be a set of indexed families which are all defined over the same index set $I$. Then, we also define $\proj_j(A)\coloneqq \{\proj_j(a)\mid a\in A\}$. Analogously, $\proj_J(A)\coloneqq \{\proj_J(a)\mid a\in A\}$. (Note that when $J$ is empty, $\proj_J(A)$ is a singleton with the empty family as its only element.) We also use the notation $A_j\coloneqq \proj_j(A)$ and $A_J\coloneqq \proj_J(A)$ as a shorthand.
\item \textbf{Merge.} Let $a=(a_j)_{j\in J}$, $b=(b_k)_{k\in K}$ be two indexed families over disjoint index sets $J$ and $K$. Then $a\merge b\coloneqq  (c_i)_{i\in J\merge K}$ with $c_i\coloneqq  a_i$ if $i\in J$ and $c_i\coloneqq b_i$ if $i\in K$. Note that the merge $a\merge b$ is similar to a concatenation of vectors. A concatenation of vectors can be seen as the special case of $a\merge b$ when $J\merge K$ is ordered and $j<k$ for all $j\in J$ and $k\in K$. Unlike concatenation, merging is commutative.
\item
\textbf{Cartesian product.} Let $(A_i)_\iinI $ be a family of sets $A_i$ over the index set $I$. Then its \emph{Cartesian product} is $\timesbig_\iinI A_i\coloneqq \{(a_i)_\iinI \mid  \forall i\in I: a_i\in A_i\}$, which is a singleton if $I$ is empty.
For a set $B$ of indexed families over $J$, and a set $C$ of indexed families over $K$, wherein $J$ and $K$ are disjoint, we write $B \times C \coloneqq  \{b \merge c\mid b\in B, c\in C\}$. Note that unlike the Cartesian product over vectors, the Cartesian product over families is  commutative because it uses the merge rather than concatenation. Also note that the projection of a Cartesian product is again a Cartesian product. That is, if $A=\timesbig_\iinI A_i$, then the projection $A_J$ equals $\timesbig_{i\in J}A_i$.
\end{enumerate}
Using these definitions, we now define factored spaces and factored space models.
\subsection{Factored spaces and factored space models}
We now define two central mathematical structures of our framework --- \emph{factored spaces} and \emph{factored space models}.

\begin{definition}[Factored Space]
    A finite sample space~$\Omg$ is called a \emph{factored space} if there is a finite \emph{index set $I$} and finite sets $\Omg_i$ for all $i\in I$ such that \(\Omg = \timesOmg\) holds.
    The sets $\Omega_i$ are called the \emph{factors} of $\Omega$.
    Moreover, the random variables $U_i\colon\Omega\to\Omega_i$ with $U_i(\omega)=\pi_i(\omega)$ are called the \emph{background variables} of $\Omega$, and we write $U=(U_i)_\iinI$.
\end{definition}
The elements of $\Omg$ are indexed families of the form $\omg = (\omg_i)_\iinI $ with $\omg_i\in \Omg_i$.
When using factored spaces for modeling probability distributions, we consider those distributions that \emph{factorize} over~$\Omg$.
\begin{definition}[Factorizing Distribution]
    \label{def:factorizes}
    Let $P$ be a probability distribution on a factored space~$\Omega$. We say that $P$ \textnormal{factorizes} over~$\Omg$ if 
    \(P(\omega)=\prod_{i\in  I}P(\omega_i)\) holds for all $\omega\in\Omega$, wherein $P(\omega_i)\coloneq P(\proj_i^{-1}(\omg_i))$.
    
\end{definition}
Note that $P$ factorizes over $\Omg$ if and only if all background variables $U_i$ are mutually independent under $P$.
As illustrated in \Cref{fig:fsm}, one can think of the factors as axes, and the sample space as a hyper-rectangle along those axes.
We remark that $U\colon\Omega\to\Omega$ is the identity function on $\Omega$, so any variable $X:\Omg\to \Val(X)$ is a deterministic function of~$U$. This is analogous to every variable in a structural causal model \cite{pearl2009causality} being a deterministic function of the independent background variables $U$.
\begin{figure}[ht]
\centering
\includegraphics[width=0.75\textwidth]{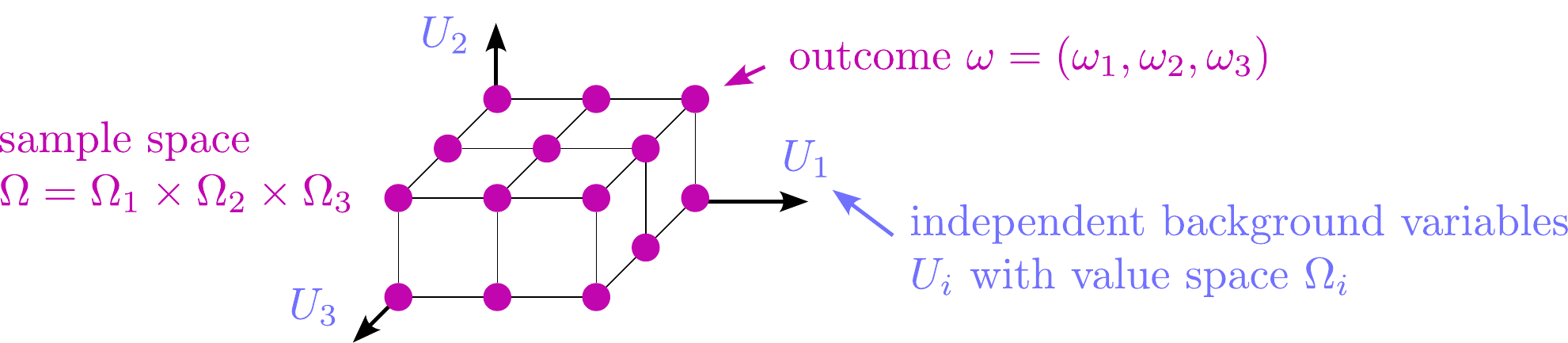} 
\caption{Factored space with three factors.}
\label{fig:fsm}
\end{figure}%

Next, we formally define the notion of a \emph{factored space model} for a probability distribution~$P$ on a finite set~$\Obs$ of possible observations.
\begin{definition}[Factored Space Model]  \label{def:FSM}
    Let $\Obs$ be a finite set and let $P$ be a distribution on $\Obs$.
    Furthermore, let $O\colon\Omega\to\Obs$ be a random variable on a factored space~$\Omega$.
    Then we say that the tuple $\FSModel\coloneqq(\Omega,O)$ is a \emph{factored space model for~$P$} if there is a distribution $P^\Omg$ that factorizes over $\Omega$ and satisfies $P^\Omg(O=o)=P(o)$ for all $o\in\Obs$.
\end{definition}
We call a variable $X$ on $\Omg$ \emph{observed} if it is a deterministic function of $O$, and \emph{unobserved} otherwise.
We remark that every probability distribution~$P$ over $\Obs$ has a trivial factored space model with a single factor (where $\Omega=\Obs$ and $|I|=1$), but it may have many other factored space models.

\subsection {History}

In the following, we build up to a notion of \textit{structural independence} of variables on factored space models which is an analog of d-separation and which applies to arbitrary variables. To this end, we first define the \textit{history} $\history(X\mid C)$ of a random variable $X$ given an event $C\subseteq \Omg$. The letter $C$ indicates that this is the event we condition on.
We would like to define the history such that it is the set of indices $i$ of the background variables $U_i$ that $X$ depends on if we condition on $C$ in a distribution that factorizes over $\Omg$.
One might think that the set of background variables that $X$ depends on should be defined as the smallest set of variables $U_J$ such that $U_J\functionTo_C X$. However, this condition is not enough, since conditioning on $C$ may make the background variables dependent. To see this, consider the following example.

Let $\Omg_1\coloneq\Omg_2\coloneq\{0,1\}$ be the outcomes of two independent coin flips, and let $P$ be the joint distribution over the factored space $\Omg=\Omg_1\times \Omg_2$. Then, the background variables $U_1$ and $U_2$ are the results of first and second coin respectively. Since the coins are independent, $P$ factorizes over $\Omg$. Further, let $C\coloneq \{00, 11\}$ be the event that both coins have the same result. Note that given $C$, we have that $U_1$ and $U_2$ are dependent in $P$. Thus, the set of background variables that $U_1$ depends on given $C$, is $\{U_1, U_2\}$, rather than just $\{U_1\}$. When $C$ is such that it does not introduce a dependence between $U_J$ and $U_{I\setminus J}$ for $J\subseteq I$, we say that $J$ \emph{disintegrates} $C$.

\begin{definition}[Disintegration]\label{def:disintegration}
Let $\FS=\timesOmg$ be a factored space, let $C\subseteq\Omg$ be an event and let $J\subseteq I$. Then $J$ \emph{disintegrates} $C$ if
    $C=C_J\times C_{I\backslash J}\,.$
\end{definition}
Note that in the previous example, $\{1,2\}$ trivially disintegrates $C$ but $\{1\}$ does not disintegrate $C$. We can now formally define the history.
\begin{definition}[History, Generation]
\label{def:history}
Let $\FS=\timesOmg$ be a factored space with the background variables~$U$. Let $J\subseteq I$, let $X\colon \Omg\to \Val(X)$, and let $C\subseteq \Omg$. Then $J$ \textnormal{generates} $X$ given $C$ if $U_J \functionTo_{C} X$ and $C=C_J\times C_{I\backslash J}$.
The \textnormal{history} $\history(X\mid C)$ of $X$ given $C$ is the intersection of all $J\subseteq I$ 
that generate $X$ given $C$. That is,
\begin{align}\label{eq:def-history}%
\history(X\mid C)\coloneqq \sectbig \big\{J \subseteq I: J  \textnormal{ generates } X \textnormal{ given } C\big\}\,.
\end{align}
\end{definition}

Note that, for any event~$C$, we have $C=C_I\times C_\emptyset$ and thus the index set $I$ trivially disintegrates $C$. Moreover, for any variable~$X$, the index set $I$ trivially generates~$X$ because $U_I$ is the identity function on $\Omg$ and $X$ is a function $\Omg\to\Val(X)$, so $X$ is derived from $U_I$ on~$C$ in the sense of \Cref{def:derived}. Therefore, the set on the right side of \eqref{eq:def-history} contains at least the set $I$.

\begin{restatable}[History is minimal generating set]{lemma}{historygenerates}
    \label{lemma:historygenerates}
    For a factored space $\FS=\timesOmg$, let $X\colon \Omg\to \Val(X)$, and let $C\subseteq\Omg$. Then, $\history(X\given C)$ is the unique minimal set which generates $X$ given $C$.
\end{restatable}
\begin{proof}[Proof (sketch)]
    We first show that generation is closed under intersection, see \Cref{lemma:generationintersection} in the appendix.
    It follows that $\history(X\given C)$ generates $X$ given $C$. As $I$ is finite, $\history(X\given C)$ is the unique minimum of $\{J\subseteq I \mid J\text{ generates } X\text{ given }C\}$. Therefore, $\history(X\given C)$ can be equivalently defined as \emph{the minimal set} that generates $X$ given $C$.
\end{proof}

\paragraph{Shorthand notation for the history}
For an event $A\subseteq \Omg$, we write $\history(A
\given C)$ as a shorthand for $\history(1_A\given C)$, wherein $1_A$ is the variable that is $1$ for $\omg \in A$, and $0$ otherwise.
When $x$ and $y$ are values of the random variables $X$ and $Y$, we use the notation $\history(x\given C), \history(X\given y), \history(A\given y)$ or $\history(x\given y)$. Here, $x$ is a shorthand for the event $\{\omg \in \Omg\mid X(\omg)=x\}$ which is a common abbreviation when denoting probabilities, such as $P(x\given y)$.
We write $\history(X), \history(A)$ and $\history(x)$ as a shorthand for the unconditional history $\history(X\mid \Omg), \history(A\given \Omg)$ and $\history(x\given\Omg)$ respectively.

We remark that the disintegration condition is vacuous for the unconditional history, as $\Omg$ always satisfies $\Omg=\Omg_J\times \Omg_{I\backslash J}$.

\paragraph{History of joint variables}
The following lemma formalizes the fact that if a variable~$X\colon\Omg\to\Val(X)$ depends on exactly the positions $J_1\subseteq I$ in $\Omg$ and a variable~$Y\colon\Omg\to\Val(Y)$ depends on exactly the positions $J_2\subseteq I$, then the joint variable~$(X,Y)$ depends on exactly the positions $J_1\cup J_2$.
\begin{restatable}[History of joint variable]{lemma}{jointhistory}\label{lemma:jointhistory}
Let $X$ and $Y$ be  variables defined on a factored space $\FS=\timesOmg$. For all events $C\subseteq \Omg$, we have
\[\history{(X,Y) \given C}=\history{X \given C}\cup \history(Y \given C)\,.\]
\end{restatable}
The proof is a straightforward application of \Cref{lemma:historygenerates}
and can be found in \Cref{sec:proofjointhistory}.
Next, we introduce a lemma which establishes the relationship between the history of a variable~$X$ and the histories of its values $x\in \Val(X)$.
\begin{restatable}[History of a variable is the union of the histories of events]{lemma}{lemmaunionhistory}%
\label{lemma:unionhistory}%
Let $X$ be a random variable defined on a factored space $\FS$. For all events $C\subseteq \Omg$, we have
\[\history{X \given C}=\bigcup_{x\in\Val(X)}\history{x \given C}\,.\]
\end{restatable}
\noindent
The proof is an inductive application of \Cref{lemma:jointhistory} and can be found in \Cref{sec:proofunionhistory}.

\subsection{Structural independence}

In the following, we define two variables~$X\colon\Omg\to\Val(X)$ and $Y\colon\Omg\to\Val(Y)$ to be structurally independent if they depend on disjoint factors of the factored space $\Omg=\bigtimes_\iinI  \Omg_i$. We use our notion of history to formalize this definition and extend it to the situation of conditioning on a third variable. This is our analog of d-separation in Bayes nets.
\begin{definition}[Structural Independence]\label{def:structuralindependence}
  Let $X$, $Y$, and $Z$ be random variables in a factored space~$\FS$. Then, $X$ and $Y$ are \emph{structurally independent} in $\FS$, denoted as $X \orthF Y$, if $\history(X) \sect  \history{Y} = \nothing$.
  Moreover, $X$ and $Y$ are \emph{structurally independent given $Z$}, denoted as
  $X \orthF Y \given Z$, if we have
  \[\history(X \mid z) \sect \history{Y \given z} = \nothing \text{ for all } z\in \Val(Z)\,.\]
\end{definition}

We choose the name \emph{structural independence} because it represents those statistical independences which come from a structure in the process that generated the distribution. For example, when throwing two fair coins $X_1$ and $X_2$, then $X_1$ is statistically independent of the XOR variable $X_\oplus=X_1\oplus X_2$, because $P(x_1, x_\oplus)=0.25=P(x_1)P(x_\oplus)$ holds for all $x_1,x_\oplus \in\{0,1\}$. However, this independence is not structural because it relies on the exact parameters of the process. If one coin is slightly unfair and $P(X_1=0)=51\%$, then $X_1$ and $X_\oplus$ are not independent anymore. In contrast, $X_1$ and $X_2$ are still independent, no matter the distribution of $X_1$ and $X_2$. In that sense, the independence of $X_1$ and $X_2$ is a property of the process which generates a class of distributions rather than a property of a particular distribution. The formal connection between statistical independence and structural independence is established in \Cref{sec:soundcomplete}, where we prove that $X$ and $Y$ are independent given $Z$ in all product probability distributions over $\FS$ if and only if $X$ and $Y$ are structurally independent given~$Z$ in~$\FS$.

\subsection{Structural time}
Let $X$ and $Y$ be two random variables. In the following, we formalize a notion of $X$ being \emph{before} $Y$ in the sense that $X$ is always determined before $Y$ in the process that generated $X$ and $Y$.
\begin{definition}[Structural Time]\label{def:structuraltime}
The \emph{structural time} $\beforeF$ of a factored space $\FS$ compares two variables $X$ and $Y$ on $\Omg$. We say that \emph{$X$ is before $Y$} in $\FS$, denoted $X\beforeF Y$, if $\history(X)\subseteq \history(Y)$.
We say that~$X$ is \emph{strictly before} $Y$, denoted $X\strictlybeforeF Y$ if $\history(X)\subsetneq\history(Y)$.
\end{definition}
We can interpret the history $\history(X)$ as the set of \emph{sources of randomness} that $X$ depends on. If $\history(X)\subseteq\history(Y)$ holds, then $X$ depends on a subset of the sources of randomness that $Y$ depends on, which fits with the intuition that $X$ is determined before $Y$. More precisely, in any process in which the values of $U_i$ become known one after the other, the value of $X$ will always be determined before the value of $Y$, or at the same time.

In \Cref{sec:dagconstruction} we show that when we construct an FSM from a causal graph, then for variables that are nodes in a causal graph, the structural time is equivalent to the ancestor relation. Structural time is more general than the ancestor relation, since it is defined on all variables on $\Omg$, not only the ones which are nodes in the graph.

In the following lemma, we show that the structural time can be equivalently expressed in terms of structural independence.
We hold that by this lemma the name \emph{structural time} is philosophically justified, but this topic is outside the scope of this paper.

\begin{restatable}[Structural time and structural independence]{lemma}{structuraltimehistory}
\label{lemma:structuraltimehistory}
Let $X$ and $Y$ be random variables on~$\FS$. Then $X$ is structurally before $Y$ if and only if
\(Y \orthF Z \impl X\orthF Z \textnormal{ holds for all variables }Z\textnormal{ on }\Omg.\)
\end{restatable}
\noindent
The lemma follows from the definition of structural independence, and the proof can be found in \Cref{sec:proofstructuraltime}.

\section{Bayesian networks and factored space models}
\label{sec:dag}

In this section, we compare factored space models (FSMs) with Bayesian Networks. In \Cref{sec:semigraphoid} we compare the properties of d-separation and structural independence in terms of the graphoid axioms they fulfill. In \Cref{sec:dagconstruction} we construct an FSM from a causal graph, and show that node variables in the graph are d-separated if and only if they are structurally independent in the corresponding FSM.

\subsection{Graphoid and semigraphoid axioms}
\label{sec:semigraphoid}

Structural independence is our analog of d-separation. In this section, we compare the properties of d-separation, which forms a \textit{compositional graphoid} and structural independence, which forms a \textit{compositional semigraphoid}. This comparison also provides an intuition as to why structural independence represents deterministic relationships correctly while d-separation does not.

\begin{definition}[Compositional Semigraphoid]
A \emph{semigraphoid} is a set of triplets $(X,Y,Z)$, usually denoted as $X\perp Y\mid Z$, that satisfy the symmetry, decomposition, weak union, and contraction axioms, as listed in \Cref{tab:semigraphoid}. A \emph{graphoid} \cite{pearl2022graphoids} is a semigraphoid that also satisfies the intersection axiom. A \emph{compositional} graphoid or semigraphoid additionally satisfies the composition axiom.
\end{definition}

\begin{table}[ht]
\resizebox{\textwidth}{!}{
\begin{tabular}{@{}lllccc@{}}
\toprule
\multicolumn{2}{l}{Axioms ({\color[HTML]{00009B} Semigraphoid: 1.-4.}, Graphoid: 1.-5.)}                                                        &                                                    & \begin{tabular}[c]{@{}c@{}}Indep. \\ (Probability)\end{tabular} & \begin{tabular}[c]{@{}c@{}}d-Sep. \\ (DAG)\end{tabular} & \begin{tabular}[c]{@{}c@{}}Str. Indep.\\ (FSM)\end{tabular} \\ \midrule
{\color[HTML]{00009B} 1. Symmetry:}      & {\color[HTML]{00009B} $X\perp Y \mid W $}                                     & {\color[HTML]{00009B} $\impl Y\perp X\mid W$}      & {\color[HTML]{32CB00} \cmark}                                   & {\color[HTML]{32CB00} \cmark}                           & {\color[HTML]{32CB00} \cmark}                               \\
{\color[HTML]{00009B} 2. Decomposition:} & {\color[HTML]{00009B} $X\perp Y, Z\mid W$}                                  & {\color[HTML]{00009B} $\impl X\perp Y\mid W$}      & {\color[HTML]{32CB00} \cmark}                                   & {\color[HTML]{32CB00} \cmark}                           & {\color[HTML]{32CB00} \cmark}                               \\
{\color[HTML]{00009B} 3. Weak Union:}    & {\color[HTML]{00009B} $X\perp Y,Z\mid W$}                                   & {\color[HTML]{00009B} $\impl X\perp Z\mid Y ,W$} & {\color[HTML]{32CB00} \cmark}                                   & {\color[HTML]{32CB00} \cmark}                           & {\color[HTML]{32CB00} \cmark}                               \\
{\color[HTML]{00009B} 4. Contraction:}   & {\color[HTML]{00009B} $(X\perp Y\mid W)\land (X\perp Z\mid Y ,W)$}            & {\color[HTML]{00009B} $\impl X\perp Y,Z\mid W$}  & {\color[HTML]{32CB00} \cmark}                                   & {\color[HTML]{32CB00} \cmark}                           & {\color[HTML]{32CB00} \cmark}                               \\
{\color[HTML]{000000} 5. Intersection:}  & {\color[HTML]{000000} $(X\perp Y\mid Z, W) \land (X\perp Z\mid Y,W) \land (Y\ne Z)$} & {\color[HTML]{000000} $\impl X\perp Y,Z\mid W$}  & {\color[HTML]{32CB00} \xmark}                                   & {\color[HTML]{CB0000} \cmark}                           & {\color[HTML]{32CB00} \xmark}                               \\
{\color[HTML]{000000} 6. Composition:}   & {\color[HTML]{000000} $(X\perp Y\mid W)\land (X\perp Z\mid W)$}               & {\color[HTML]{000000} $\impl X\perp Y,Z\mid W$}  & {\color[HTML]{32CB00} \xmark}                                   & {\color[HTML]{CB0000} \cmark}                           & {\color[HTML]{CB0000} \cmark}                               \\ \midrule
                                         &                                                                               &                                                    & \multicolumn{3}{c}{{\color[HTML]{32CB00} green: same as independence}}                                                                                                                  \\
                                         &                                                                               &                                                    & \multicolumn{3}{c}{{\color[HTML]{CB0000} red: different from independence}}                                                                                                             
\end{tabular}
}
\caption{\label{tab:semigraphoid}%
Comparison of \textbf{independence} in a probability distribution, \textbf{d-separation} in a directed acyclic graph (DAG), and \textbf{structural independence} on an FSM with regard to the graphoid and composition axioms. Independence satisfies 1.-4., which means it forms a \emph{semigraphoid} \cite{pearl1988probabilistic}, d-separation satisfies 1.-6., which means it forms a \emph{compositional graphoid}. Structural independence satisfies 1.-4. and 6., which means it forms a \emph{compositional semigraphoid}. Note that in terms of these axioms, structural independence is closer to independence than d-separation, as d-separation differs from independence for both the intersection and the composition axiom, while for structural independence the only difference is the composition axiom.}
\end{table}

\begin{restatable}{proposition}{compositionalsemigraphoid}
\label{prop:compositionalsemigraphoid}
Structural independence is a compositional semigraphoid.
\end{restatable}
The proof is straightforward and we defer it to \Cref{sec:proofcomposition}.

\paragraph*{Composition} The composition axiom $(X\orth Y\mid W)\land(  X\orth Z\mid W)\impl (X\orth Y,Z\mid W)$ illustrates that structural independence is not the same as independence: Structural independence satisfies this axiom intuitively, because if $X$ and $Y$ are influenced by disjoint sources of randomness, and $X$ and $Z$ are influenced by disjoint sources of randomness, then $X$ and $Y,Z$ are influenced by disjoint sources of randomness as well. To see that statistical independence does not satisfy the composition axiom, consider the following example:
Let $X$ be a message, $Y$ be a string of uniformly distributed bits, and let $Z$ be their bitwise XOR: $Z=X\oplus Y$. Then $X$ and $Y$ are independent, and $X$ and $Z$ are independent, but $X$ and $(Y,Z)$ are dependent, so composition does not hold. This situation could be modeled with an FSM as follows. The variables $X$ and $Y$ are structurally independent, but as $Z$ is computed from $X$, the histories of $X$ and $Z$ overlap and $X$ and $Z$ are structurally dependent, so composition is not violated. Independence is about the question ``Does $X$ provide information about $Y$?'' while structural independence is about the question ``Are $X$ and $Y$ influenced by the same source of randomness?''.

\paragraph*{Intersection} The intersection axiom $(X\perp Y\mid  Z,W)\land (X\perp Z\mid Y,W)\land (Y\ne Z)\impl (X\perp Y,Z\mid W)$ is particularly interesting. Statistical independence does not in general satisfy the intersection axiom, though intersection does hold in those cases where $P(x,y,z,w)$ is never zero \cite{pearl1988probabilistic}. However, we \textit{want} to allow $P(x,y,z,w)=0$. In particular, if any variable is a deterministic function of another variable, there must be values $x,y,z,w$ with $P(x,y,z,w)=0$. The fact that d-separation satisfies intersection means that d-separation cannot represent the independence relations among a set of variables, some of which are deterministically related. It is therefore an important property of structural independence that the intersection axiom does \textit{not} hold.

In the next section, we continue our comparison of d-separation and structural independence by constructing an FSM from a causal graph.
\subsection{From Bayesian networks to factored space models}
\label{sec:dagconstruction}

In this section, we construct a factored space from a directed acyclic graph (DAG), and we construct an FSM from a Bayesian network. We show that a distribution $P$ factorizing over a DAG is equivalent to our construction being a factored space model of $P$.
We also show that in our construction, structural independence applied to node variables is equivalent to d-separation, and structural time applied to node variables is equivalent to the ancestor relationship. We also introduce the notion of a \emph{perfect map}, and use it to show that factored space models are more expressive than Bayesian networks. We start by formally introducing Bayesian networks.

\paragraph*{Bayesian networks}
Let $G=(V,E,\Val)$ be a directed acyclic graph (DAG), wherein $V$ is the vertex set, and $E$ is the edge set. Moreover, each vertex $v\in V$ is associated with a set of possible values $\Val_v$ with at least $2$ elements, and $\Val=\timesbig_\vinV \Val_v$ is the set of all combinations of values for $V$.

For a vertex~$v$, we write $\pav\subseteq V$ for the set of parents of $v$, and $\A(v)$ for the set of ancestors of $v$.
If $P$ is a distribution over $\Val$ and $x_v\in \Val_v$ is a value of $v$, we write $P(x_v)$ as a shorthand for the probability $P(\{(y_u)_{u\in V}  \in \Val_V \colon x_v=y_v\})$.
We say that $P$
\emph{factorizes} over $G$, if for all $x=(x_v)_\vinV \in \Val$, we have
\begin{align}\label{eq:Bayes-factorizes}%
P(x)=\prod_\vinV P(x_v\given \xpav)\,.
\end{align}
Following the notation in \cite{koller2009probabilistic}, a pair $\Bayes=(G, P)$ where $P$ factorizes over $G$ is called Bayesian network or Bayes net.
In the following, we show how to construct a factored space $\FSG$ and a factored space model $\FSMG$ from $G$.
Factors in this construction are closely related to factor graphs \cite{witten2002data}.

\subsubsection{Our construction of a factored space model from a Bayesian network}
  
Given the DAG $G=(V,E,\Val)$,
we now construct a factored space model $\FSMG=(\FSG,O)$ for the observation space $\Obs=\Val$.
We first construct a factored space $\FSG$.
To this end, we define the index set~$I$ via $I\coloneqq\bigcup_\vinV  I_v$, where for all $v\in V$, we let
\begin{align}
I_v \coloneqq  \{(v,\xpav) \;:\;  \xpav \in
\Val_\pav\}\,.
\end{align}
In other words, the index set $I$ is partitioned into the sets $I_v$, and each element of $I_v$ corresponds to a possible value for the vertices in $\pav$.
If $v$ is a root, we remark that $\pav=\nothing$ holds and that
$\Val_\pav$
contains the empty tuple $()$ as its only element.
Finally, for each $i=(v, \xpav)\in I$, we define the set $\FSG_i$ via $\FSG_i\coloneqq\Val_v$. Then $\FSG\coloneqq\timesbig_\iinI \FSG_i$ is the factored space constructed from $G$.

We now define the random variable $O\colon \Omg\to\Val$.
To this end, we first define $\Xb=(X_v)_\vinV $, wherein $X_v\colon \Omg\to\Val_v$ is a random variable associated with each node $v$.
We define $X_v$ recursively as
\begin{align}\label{eq:def-graph-X}
    X_v(\omg)\coloneqq \omg_{(v, \Xpav(\omg))}\,.
\end{align}
In particular, for roots~$v$ we have $X_v(\omg)=\omg_{(v,())}\in\Val_v$.
Since $G$ is acyclic, $X_v$ is well-defined also for non-roots~$v$. Note that $\Val(\Xb)=\Val$ and $\Val(X_v)=\Val_v$.
We call $\FSMG=(\FSG, O)$ with $O=\Xb$ the \emph{factored space model constructed from $G$}. This concludes the construction.

To show that $\FSMG$ is in fact a factored space model of a distribution $P$ over $\Val$ in the sense of \Cref{def:FSM} when $P$ factorizes over $G$, we first state the following lemma about the relationship between $G$ and $\FSMG$.

\begin{restatable}{lemma}{taubijective}
    \label{lemma:taubijective}
Let $\Delta^*(G)$ be the set of distributions on $\Val$ that factorize over $G$ and let $\distributionsF(\Omg)$ be the set of distributions on~$\Omega$ that factorize over $\Omg$.
Let $\tau\colon\distributionsF(\Omg)\to \Delta^*(G)$ be the function defined as follows for all $P^\Omg\in \distributionsF(\Omg)$ and $x\in \Val$:
\begin{align*}
  \tau(P^\Omg)(x)
  &=P^\Omg(X=x)
  \,.\\
  \intertext{Then $\tau$ is bijective and its inverse $\tau^{-1}$ satisfies the following for all $P\in \Delta^*(G)$ and $\omg\in\Omg$:}
  \tau^{-1}(P)(\omg)
  &=
  \prod_{(v,x_{\pa(v)})\in I} P(\omg_{v,x_{\pa(v)}}\given x_{\pa(v)})\,.
\end{align*}
\end{restatable}
The proof can be found in \Cref{sec:proofdistributionconstruction}.

\subsubsection{Properties of our construction}
In this section, we fix a DAG $G=(V,E,\Val)$, the corresponding factored space model $\FSMG=(\FSG, O)$ constructed from $G$, and the node variables~$X_v$ for $v\in V$.
We show that our construction preserves factorization, d-separation, and the ancestor relation.
We start with the factorization property.
\begin{proposition}[Factorization property]
\label{prop:distributionconstruction}
    For all distribution~$P$ over $\Val$, we have that $P$ factorizes over $G$ if and only if $\FSMG$ is a factored space model of $P$.
\end{proposition}

\begin{proof}
  For the forward implication, let $P$ factorize over $G$, that is $P\in \Delta^*(G)$. Then, with $\iota$ defined as in \Cref{lemma:taubijective}, let $P^\Omg\coloneq\iota(P)$. By \Cref{lemma:taubijective}, $\tau$ is the inverse of $\iota$. Therefore, for all $x\in \Val$, we have that $P(x)=\tau(\Pomg)(x) = \Pomg(X^{-1}(x)) =P^\Omg(X=x)$. Since $\Pomg$ also factorizes over $\Omg$ by the definition of $\iota$,
$\FSMG$ is an FSM of $P$ in the sense of \Cref{def:FSM}.

For the backward implication, let $\FSMG$ be an FSM of $P$. That is, there is a distribution $P^\Omg$ that factorizes over $\Omg$ such that $\Pomg(\Xb=x)=P(x)$. That means $P=\tau(\Pomg)$. By \Cref{lemma:taubijective}, $\Pomg=\iota(P)$, and therefore $P$ factorizes over $G$ by the definition of $\iota$.
\end{proof}

Next, we show how d-separation in $G$ is preserved as structural independence in~$\FSMG$.
\begin{restatable}[d-separation]{proposition}{dseparationstructuralindependence}
\label{prop:dseparation-structuralindependence}
Let $V_1,V_2,V_3\subseteq V$ be three sets of nodes in $G$. Then $V_1$ and $V_2$ are d-separated given $V_3$ in $G$ if and only if $X_{V_1}$ and $X_{V_2}$ are structurally independent given $X_{V_3}$ in $\FSMG$.
\end{restatable}
We defer the straightforward proof to \Cref{sec:proof-d-sep-str-indep}.

Finally, we show how the ancestor relation in~$G$ is preserved as the structural time relation between the node variables.
\begin{proposition}[Ancestor relation]
Let $v_1,v_2\in V$ be two distinct nodes in the graph~$G$. Then $v_1$ is an ancestor of $v_2$ in $G$ if and only if $X_{v_1}\strictlybefore^{\FSG} X_{v_2}$ holds.
\end{proposition}
\begin{proof}
The history of $X_v$ is $\history(X_v)=\bigcup_{u\in\A(v)\cup \{v\}}I_{u}$. Since $I_v$ is nonempty, we have that $v_1$ is an ancestor of $v_2$ $\iff$ $\A(v_1)\cup\{v_1\}\subsetneq\A(v_2)\cup\{v_2\} $  $\iff$ $\bigcup_{v\in\A(v_1)\cup \{v_1\}}I_v \subsetneq \bigcup_{v\in\A(v_2)\cup \{v_2\}}I_v$  $\iff$ $\history(X_{v_1})\subsetneq\history(X_{v_2})$,
which is exactly the definition of $X_{v_1}<^\FS X_{v_2}$.
\end{proof}

\subsubsection{Factored space models are more expressive than Bayesian networks}
A DAG whose d-separations capture the independences of a distribution is called a \emph{perfect map} \cite{koller2009probabilistic}. We now define our analog. That is, an FSM whose structural independences capture the independences of a distribution. With this definition, we can show that FSMs are more expressive than DAGs, in that there are distributions $P$ such that there is no perfect map DAG of $P$ while there is a perfect map FSM of $P$.

\begin{definition}[Perfect map] Let $G$ be a DAG with nodes $V$ and value space $\Val(\Xb)$. Further, let $\FSModel=(\FS, O)$ 
be a factored space model. Let $P$ be a distribution over some observation space $\Obs$. Then we say that
\begin{enumerate}
    \item $G$ is a \emph{perfect map} of $P$ if  $\Val(\Xb)=\Obs$ and for all sets of nodes $V_1, V_2, V_3\subseteq V$, the sets $V_1$ and $V_2$ are d-separated given $V_3$ in $G$ if and only if $X_{V_1}$ and $X_{V_2}$ are independent given $X_{V_3}$ in $P$.
    \item $\FSModel$ is a \emph{perfect map} of $P$ with regard to a family of variables $X=(X_w)_{w\in W}$ with $X_w\colon \Omg \to \Obs$ if $\FSModel$ is a factored space model of $P$ and for all sets of variables $X_{W_1},X_{W_2}, X_{W_3} \subseteq X$, we have that $X_{W_1}$~and~$X_{W_2}$ are independent given $X_{W_3}$ in $P$ if and only if $X_{W_1}$ and $X_{W_2}$ are structurally independent given $X_{W_3}$ in~$\FS$.
\end{enumerate}
\end{definition}
Perfect maps are often considered to be the graphs that correspond to a distribution. The following proposition shows that factored spaces are more expressive than DAGs.

\begin{proposition}[Perfect maps of graphs and factored spaces]
\label{prop:perfectmap}
    Let $P$ be a distribution over some observation space $\Obs$ of the form $\Obs=\timesbig_\vinV \Val(X_v)$. Then, it holds that
    \begin{enumerate}
        \item  If there is a DAG $G$ with nodes $V$ that is a perfect map of $P$, then there is a factored space model $\FSModel=(\Omg, X)$, with $X=(X_v)_\vinV$ that is also a perfect map of $P$ with regard to $X$.
        \item If there is a factored space model $\FSModel$ that is a perfect map of $P$ with regard to $(X_v)_\vinV $, there may still be no DAG $G$ with nodes $V$ that is a perfect map of $P$.
    \end{enumerate}
\end{proposition}
We defer the proof to \Cref{sec:proofperfectmap}.
\Cref{prop:perfectmap} illustrates that factored space models are more expressive than Bayesian networks. We can construct a factored space model $\FSMG$ for a DAG~$G$, which is a perfect map for the same distributions as $G$. However, it is not possible in general to construct a DAG~$G$ from a factored space model $\FSModel$ in such a way that $G$ is a perfect map for the same distributions as $\FSModel$.

In the following section, we show that the soundness and completeness theorem of d-separation~\cite{koller2009probabilistic} generalizes to this additional expressiveness. That is, we show the soundness and completeness of structural independence.

\section{Soundness and Completeness of Structural Independence}
\label{sec:soundcomplete}

This section contains our main result: random variables are conditionally independent in all probability distributions that factorize over a factored space if and only if they are structurally independent.
To state the theorem formally, we first review the definition of conditional independence.

\begin{definition}
    \label{def:conditionalindependence} Let $P$ be a probability distribution  over a sample space~$\Omg$, and let $A,B,C\subseteq\Omg$ be events. We say that $A$ and $B$ are \emph{independent given} $C$, denoted $A\indep^P B\given C$, if the following holds:
    \begin{align}\label{eq:conditional-independence}
      P(A\cap C)P(B\cap C)=P(A\cap B\cap C)P(C)\,.
    \end{align}
If $x,y,z$ are values of the random variables $X,Y,Z$, then we write $x\indep^P y\given z$ if $P(x,z)P(y,z)=P(x,y,z)P(z)$ holds. For random variables~$X,Y,Z$, we write $X\indep^P Y\given Z$ if $x\indep^P y\given z$ holds for all values $x,y,z$ of $X,Y,Z$.
\end{definition}
\Cref{eq:conditional-independence} is equivalent to the more common definition $P(A\given C)P(B\given C)=P(A\cap B\given C)$ if $P(C)>0$.
If $P(C)=0$, our convention is that $A\indep^P B\given C$ holds and indeed \Cref{eq:conditional-independence} is trivially true.

\subsection{Main theorem}
\label{sec:theorem}
For random variables $X,Y,Z$ over a factored space $\Omg$, how does probabilistic independence $X\indep^P Y\given Z$ (\Cref{def:conditionalindependence}) relate to structural independence $X \orthF Y \given Z$ (\Cref{def:structuralindependence})?
We answer this question as follows:
Structural independence implies probabilistic independence for all distributions~$P$ that factorize over the factored space (soundness). Conversely, probabilistic independence in all distributions~$P$ that factorize over the factored space implies structural independence (completeness).

\begin{theorem}[Soundness and Completeness of Structural Independence]\label{thrm:soundcomplete}%
Let $X,Y,Z$ be random variables on a factored space $\FS=\timesOmg$. Then the following statements are equivalent.
\begin{enumerate}[label=(\roman*),leftmargin=*]
    \item\label{item:maintheorem-structural} $X$ and $Y$ are structurally independent given $Z$ in $\FS$.
    \item\label{item:maintheorem-probabilistic} $X \indep^P Y \given Z$ holds for all probability distributions $P$ that factorize over $\FS$.
\end{enumerate}
\end{theorem}
In order to prove \Cref{thrm:soundcomplete}, we first prove the soundness and completeness of structural independence for \emph{events} in $\Omg$, instead of random variables.
It is then easy to lift these results to random variables.
For brevity, in what follows we write $\distributionsF(\Omg)$ for the set of all distributions that factorize over $\Omg$.
The soundness direction (\ref{item:maintheorem-structural} $\implies$ \ref{item:maintheorem-probabilistic}) of \Cref{thrm:soundcomplete} for events is given by the following lemma.
\begin{restatable}[Soundness for Events]{lemma}{soundevent}%
  \label{lemma:soundevent}%
  Let $A,B,C\subseteq \Omg$ be events in a factored space $\FS=\timesOmg$.\\
    If the histories satisfy
    $\history(A\given C)\cap \history(B\given C)=\nothing$,
    then $A \indep^P B \given C$ holds for all $P\in\distributionsF(\Omg)$.
\end{restatable}
\noindent The proof is a straightforward transformation of definitions and can be found in \Cref{sec:soundevent}.
The completeness direction (\ref{item:maintheorem-probabilistic} $\implies$ \ref{item:maintheorem-structural}) of \Cref{thrm:soundcomplete} for events is given by the following lemma.
\begin{restatable}[Completeness for Events]{lemma}{completeevent}%
  \label{lemma:completeevent}%
  Let $A,B,C\subseteq \Omg$ be events in a factored space $\FS=\timesOmg$.\\
  If $A \indep^P B \given C$ holds for all $P\in\distributionsF(\Omg)$, then the histories satisfy $\history(A\given C)\cap \history(B\given C)=\nothing$.
\end{restatable}
\noindent The proof can be found in \Cref{sec:completeevents} and is the most technical contribution of this paper.
In a sense that we make formally precise, we prove that the history contains exactly those indices of $I$ that are probabilistically relevant for $A$ given $C$.
We show that, if $A \indep^P B \given C$ holds for all distributions $P\in\distributionsF{\Omg}$, then no index~$i\in I$ can be relevant to both~$A$ given~$C$ and to $B$ given~$C$.
Since the history contains exactly the relevant indices, this implies $\history{A\given C} \cap \history{B\given C} = \emptyset$, which establishes \Cref{lemma:completeevent}.
Using \Cref{lemma:soundevent,lemma:completeevent}, the proof of \Cref{thrm:soundcomplete} is straightforward.
\begin{proof}[Proof of \Cref{thrm:soundcomplete}]
For all random variables~$X,Y,Z$ on $\FS$, we have:
\begin{align*}
  & X \orthF Y \given Z \\
  \Longleftrightarrow \quad
  & \forall z \in \Val(Z): \ \ \history{X \given z} \cap \history{Y \given z} = \emptyset &\text{(\Cref{def:structuralindependence})} \\
  \Longleftrightarrow \quad
  & \forall z \in \Val(Z): \ \ \bigcup_{x\in\Val(X)}\history{x\given z} \cap \bigcup_{y\in\Val(Y)}\history{y\given z} =\nothing & \text{(\Cref{lemma:unionhistory})}\\
  \Longleftrightarrow \quad
  & \forall (x, y, z) \in \Val(X) \times \Val(Y) \times \Val(Z): \ \ \history{x \given z} \cap \history{y \given z} = \nothing \\
  \Longleftrightarrow \quad
  & \forall P \in \distributionsF(\Omg): \ \forall (x, y, z) \in \Val(X) \times \Val(Y) \times \Val(Z): \ \ x \indep^P y \given z & \text{(\Cref{lemma:soundevent,lemma:completeevent})} \\
  \Longleftrightarrow \quad
  & \forall P \in \distributionsF(\Omg): \ \ X \indep^P Y \given Z\,.&\text{(\Cref{def:conditionalindependence})}
  \end{align*}
This completes the proof.
\end{proof}

\subsection{Extension: Strong completeness}

Our main theorem, \Cref{thrm:soundcomplete}, states that the structural independences $X \orthF Y \given Z$ in a factored space~$\Omg$ are exactly those triples $(X,Y,Z)$ which satisfy the conditional independences $X \indep^P Y \given Z$ for \emph{all} distributions~$P\in\distributionsF(\Omg)$.
We now strengthen this result by showing that it is sufficient to exhibit the conditional independences $X \indep^P Y \given Z$ ``locally'' for all distributions~$P$ from \emph{any non-empty open subset $S\subseteq\distributionsF(\Omg)$}.
Here, we view the set $\distributionsF(\Omg)$ as a subset of $\mathbb{R}^\Omg$ with the Euclidean topology.

To strengthen the completeness result, we first show that if a conditional independence holds locally around some distribution~$Q$, then it holds globally on~$\distributionsF(\Omg)$.
We write $d:\distributionsF{\factors}^2\to \mathbb{R}$ for the Euclidean distance on $\distributionsF{\factors} \subseteq \mathbb{R}^{\Omega}$.

\begin{lemma}[From local to global conditional independence]
    \label{lemma:openset}
    Let $X,Y,Z$ be random variables in a factored space $\FS=\timesOmg$.
    Let $Q\in \distributionsF{\factors}$ be a distribution.
    If $X \indep^{Q'} Y \given Z$ holds for all $Q'$ with $d(Q,Q')<\epsilon$, then $X \indep^{P} Y \given Z$ holds for all  $P\in \distributionsF{\Omg}$.
\end{lemma}

  The proof is based on the fact that the independence statement has the form of a polynomial, and any polynomial which is zero in an $\epsilon$-environment must be the zero polynomial. It can be found in \Cref{sec:local_to_global}.

The following proposition strengthens completeness:
It is a priori not necessary that the probabilistic independence holds for \emph{all} distributions that factorize over $\Omg$. Rather, if a probabilistic independence holds for any non-empty open set of distributions that factorize over $\Omg$, then this already implies structural independence.

\begin{proposition}[Strong Completeness]
    \label{prop:strongcomplete}
      Let $X,Y,Z$ be random variables in a factored space~$\FS$. If there exists a nonempty open set $S\subseteq \distributionsF{\factors}$ with $X \indep^{P} Y \given Z$ for all $P\in S$, then $X \orthF Y \given Z$.
\end{proposition}

\begin{proof}
  The proposition follows directly from  \Cref{thrm:soundcomplete} and \Cref{lemma:openset}.
\end{proof}

\section{Discussion and conclusion}
\label{sec:discussion}

We have introduced the factored space model framework, and defined \emph{structural independence}, which models statistical independence between variables. 
Structural independence improves upon d-separation in that it is applicable to variables with deterministic as well as probabilistic relationships. 
This improvement is useful for modeling systems at different levels of abstraction, which introduces deterministic relationships. 
We think that having an independence criterion that is applicable \textit{between} levels of abstraction, rather than only within one level of abstraction, is a very natural representation. More speculatively, we believe that it may help us understand self-referencing systems because there may not even be discrete levels of abstraction. For example, some language models can already be described as thinking about their own thinking process \cite{kadavath2022language} \cite{didolkar2024metacognitive}, which in some sense means that high-level summaries of thoughts are causally influencing object-level thoughts.

Further, we think that a theory of causality that applies across layers of abstraction may be a useful tool to understand agents using the \emph{intentional stance}, that is as systems that are best described in terms of their goals \cite{dennett1989intentional}, or as systems whose actions are caused by their model of the consequences \cite{kenton2023discovering}  \cite{everitt2021agent}.

\paragraph*{Author contributions}

The contributions of the authors include, but are not limited to, the following: 

\begin{itemize}
    \item \textbf{S. Garrabrant} originally conceived of and developed the FSM framework, and found the first proof of the soundness and completeness theorem, as presented in \cite{garrabrant2021temporal}.
    
     \item \textbf{M. G. Mayer} reframed the ideas from \cite{garrabrant2021temporal} to the formalism presented in this paper and developed the more elegant proof of the soundness and completeness theorem presented in \Cref{app:soundcompleteproof}.
    
    \item \textbf{M. Wache} clarified the relationship between Bayes Nets and FSM, and wrote most of this paper, including developing the overall structure, making connections to the existing literature, developing examples, creating figures, and formalizing proofs. 

    \item \textbf{L. Lang} developed the exposition of the proof for the soundness and completeness theorem.
    \item \textbf{S. Eisenstat} developed precursor ideas and examples that grew into the theory of FSMs, and collaborated with Scott Garrabrant to work out the ideas further.
    \item \textbf{H. Dell} provided guidance and regular feedback and edited the paper in detail.
\end{itemize}

\appendix
\newpage
\section{Proofs for Section \ref{sec:fsm}}
\label{sec:proofsfsm}

\subsection{History as the minimal generating partition}
\label{sec:historyproof}
In this section, we prove that the history $\history(X\given C)$ is the unique minimal partition that generates $X$ given $C$. We start by proving that disintegration is closed under intersection and union. For a subset $J\subseteq I$, we use $\comp{J}$ as a shorthand for $I\setminus J$.

\begin{lemma}[Disintegration closed under intersection and union]
\label{lemma:disintegrationintersection}
In a factored space $\FS=\timesbig_\iinI\Omg_i$, let $C\subseteq\Omg$ and $J,K\subseteq I$ such that $J$ disintegrates $C$ and $K$ also disintegrates $C$. Then, $J\cap K$ disintegrates $C$, and $J\cup K$ disintegrates $C$ as well.
\end{lemma}
\begin{proof}
 Let $J,K\subseteq I$ be such that $J$ disintegrates $C$ and $K$ disintegrates $C$. That is, we have
\(C=C_J\times C_\comp{J}\text{ and } C=C_K\times C_\comp{K}\,.\)
We claim that $C= C_{J\sect K}\times C_\comp{J\sect K}$ and $C= C_{J\cup K}\times C_\comp{J\cup K}$ hold as well.

We first make the following observation:
\begin{enumerate}
  \item[$(\ast)$]
  \text{For all $L\subseteq I$, we have $C \subseteq C_L\times C_{\comp{L}}$.}
\end{enumerate}
To see this, let $c\in C$, then we have $c=c_L\cdot c_{\comp{L}}$, and both $c_L\in C_L$ and $c_{\comp{L}}\in C_{\comp{L}}$ hold as required.
Then $(\ast)$ implies $C\subseteq C_{J\sect K}\times C_\comp{J\sect K}$ and $C\subseteq C_{J\cup K}\times C_\comp{J\cup K}$, so it remains to prove the reverse inclusions.

To this end, we let $C^\times\coloneqq C_{J\sect K}\times C_{\comp{J}\sect K}\times  C_{J\sect \comp{K}}\times C_{\comp{J}\sect \comp{K}}$ and claim $C^\times\subseteq C$.
Let $c\in C^\times$. By the definition of $C^\times$, there exist $c^1, c^2, c^3, c^4\in C$ such that the following holds: 
\begin{align*}
c &= c^1_{J \sect K}  \merge c^2_{\comp{J} \sect K} \merge c^3_{J \sect \comp{K}}\merge  c^4_{\comp{J} \sect \comp{K}}
 = (c^1_J)_{J \sect K}  \merge (c^2_\comp{J})_{\comp{J} \sect K} \merge (c^3_J)_{J \sect \comp{K}}\merge  (c^4_\comp{J})_{\comp{J} \sect \comp{K}} \\
 &= (c^1_J\merge c^2_\comp{J})_{(J \sect K)\cup(\comp{J} \sect K)} \merge (c^3_J\merge  c^4_\comp{J})_{(J\sect\comp{K})\cup(\comp{J} \sect \comp{K})}
= (c^1_{J} \merge c^2_\comp{J})_K \merge (c^3_J \merge  c^4_\comp{J})_\comp{K}\,.
\end{align*}
Since $J$ disintegrates $C$ and $c^1, c^2\in C$, we have $c^1_{J} \merge c^2_\comp{J}\in C$. Analogously, $c^3_J \merge  c^4_\comp{J}\in C$. From this, and the fact that $K$ disintegrates $C$, we obtain $c\in C$. Therefore, we have $C^\times \subseteq C$.

By $(\ast)$, we have
$C\subseteq C_{J\sect K}\times C_\comp{J\sect K}\subseteq C^\times$ and
$C\subseteq C_{J\sect K}\times C_\comp{J\sect K}\subseteq C^\times$.
Combined with $C^\times \subseteq C$, we have proved $C= C_{J\sect K}\times C_\comp{J\sect K}$ and $C= C_{J\cup K}\times C_\comp{J\cup K}$ as required.
\end{proof}

\begin{lemma}[Generation closed under intersection]
\label{lemma:generationintersection}
In a factored space $\FS=\timesOmg$, let $X\colon \Omg\to \Val(X)$, let $C\subseteq\Omg$, and let $J,K\subseteq I$ be such that $J$ generates $X$ given $C$ and $K$ generates $X$ given $C$. Then, $J\cap K$ generates $X$ given $C$ as well.
\end{lemma}

\begin{proof}
    Let  $J,K\subseteq I$ be such that $J$ generates $X$ given $C$ and $K$ generates $X$ given $C$. In particular, $J$ and $K$ both disintegrate $C$. By \Cref{lemma:disintegrationintersection}, it follows that $J\sect K$ also disintegrates $C$.
    As both $J$ and $K$ generate $X$ given $C$, we have
    \(U_J\functionTo_C X \text{ and }U_K\functionTo_C X\).
    We claim that $U_{J \sect K} \functionTo X$.
    Let $\omg, \omg' \in C$ s.t. $\omg_{J \sect K} = \omg'_{J \sect K}$.
    Since $J$ generates $X$ given $C$, we have
    $\omg^1 = \omg_{J} \cdot \omg'_{\comp{J}} \in C$
    and $X(\omg) = X(\omg^1)$.
    Since $J \union K$ disintegrates $C$, we have
    $\omg^2 = \omg_{J \sect K} \cdot \omg'_{I \without (J \sect K)} \in C$.
    Because $\omg^1$ and $\omg^2$ agree on $J$, we have
    $X(\omg^1) = X(\omg^2)$.
    Now, since $K$ generates $X$ given $C$, we have
    $\omg^2 = \omg_{J \sect K} \cdot \omg'_{I \without (J \sect K)} \in C$
    and $X(\omg^1)=X(\omg^2)$.
    Since $\omg_{J \sect K} =\omg'_{J \sect K}$, we have
    $\omg^2 = \omg$ and therefore $X(\omg) = X(\omg')$

\end{proof}

Using \Cref{lemma:generationintersection}, we now prove \Cref{lemma:historygenerates}, which we restate here for convenience.

\historygenerates*

\begin{proof}
    By definition, $\history(X\given C)=\bigcap\{J\subseteq I \mid J\text{ generates } X\text{ given }C\}$. As generation is closed under intersection by \Cref{lemma:generationintersection}, it follows that $\history(X\given C)$ generates $X$ given $C$. As $I$ is finite, $\history(X\given C)$ is the unique minimum of $\bigcap\{J\subseteq I \mid J\text{ generates } X\text{ given }C\}$.
\end{proof}

\subsection{History of joint variable}
\label{sec:proofjointhistory}
We prove \Cref{lemma:jointhistory}, which we restate for convenience.

\jointhistory*
\begin{proof}[Proof of \Cref{lemma:jointhistory}]
     We first prove $\history{(X,Y) \given C}\subseteq\history{X \given C}\cup \history(Y\given C)$. We begin by observing that $\history(Y\given C)\cup \history(Y\given C)$ disintegrates $C$ because disintegration is closed under union by \Cref{lemma:disintegrationintersection}.
     Moreover, by definition of the history, we have
     \(U_{\history(X\given C)}\functionTo_C X \textnormal{ and } U_{\history(Y\given C)}\functionTo_C Y\).
     Thus, we have
     \((U_{\history(X\given C)}, U_{\history(Y\given C)})\functionTo_C (X,Y)\),
     which is equivalent to
     \(U_{\history(X\given C)\cup\history(Y\given C)}\functionTo_C (X,Y)\).
     Together with the fact that $\history(X\given C)\cup \history(Y\given C)$ disintegrates $C$,
     this means that $\history(X\given C)\cup\history(Y\given C)$ generates $(X,Y)$ given $C$. By \Cref{lemma:historygenerates}, the history is the smallest generating set. Therefore, we have
     $\history((X,Y)\given C) \subseteq \history(X\given C)\cup \history(Y\given C)$.

     Now we prove the reverse inclusion, that is, $\history(X\given C)\cup \history(Y\given C)\subseteq \history((X,Y)\given C)$.
     By definition of the history, we have
     \(U_{\history((X,Y)\given C)}\functionTo_C (X,Y)\),
     and therefore also
     \(U_{\history((X,Y)\given C)}\functionTo_C X\)
     and
     \(U_{\history((X,Y)\given C)}\functionTo_C Y\).
     Together with the fact that $\history((X,Y)\given C)$ disintegrates $C$
     by definition, this means that $\history((X,Y)\given C)$ generates $X$ given $C$, and $\history((X,Y)\given C)$ also generates $Y$ given $C$. As the history is the smallest generating set by \Cref{lemma:historygenerates}, we have
     \(\history(X\given C) \subseteq \history((X,Y)\given C)\) and $\history(Y\given C) \subseteq \history((X,Y)\given C)$,
     and therefore also
     $\history(X\given C)\cup \history(Y\given C) \subseteq \history((X,Y)\given C)$.
     This completes the proof.
\end{proof}

\subsection{History as union of histories of events}
\label{sec:proofunionhistory}
Here we prove \Cref{lemma:unionhistory}, which we restate for convenience.
\lemmaunionhistory*
\begin{proof}
  We first recall that $\history(x\given C)$ is a shorthand for $\history(1_x\given C)$, where $1_x$ is the indicator random variable defined via $1_x(\omega)=1$ if $X(\omega)=x$ and $1_x(\omega)=0$ otherwise.
  We observe that each~$1_x$ is a deterministic function of $X$ and that $X$ is a deterministic function of the joint variable $(1_x)_{x\in\Val(X)}$. In particular, we have $X\functionTo_C (1_x)_{x\in\Val(X)}\functionTo_C X$.
  By transitivity of $\functionTo_C$, this implies $\history(X\given C) = \history((1_x)_{x\in\Val(X)}\given C)$.
  Since $\Val(X)$ is finite, we can inductively apply \Cref{lemma:jointhistory} to arrive at
  $\history((1_x)_{x\in\Val(X)}\given C)=\bigcup_{x\in\Val(X)}\history{1_x \given C}=\bigcup_{x\in\Val(X)}\history{x \given C}$.
  Thus, the claim follows.
\end{proof}

\subsection{Structural time is inclusion}
\label{sec:proofstructuraltime}
Here we prove \Cref{lemma:structuraltimehistory}, which we restate for convenience.

\structuraltimehistory*

\begin{proof}
  Let $X$ and $Y$ be random variables in a factored space $\FS$. We prove the two directions of the equivalence separately.

  ``$\Rightarrow$'': Suppose that $X$ is structurally before $Y$. By \Cref{def:structuraltime}, this means $\history(X)\subseteq\history(Y)$.
  We claim that \(Y \orthF Z \impl X\orthF Z\) holds for all variables \(Z\) on \(\Omg\). To this end, let $Z$ be any variable on $\Omg$ and suppose $Y\orthF Z$ holds.
  By \Cref{def:structuralindependence}, we have $\history(Y)\cap\history(Z)=\nothing$.
  By $\history(X)\subseteq\history(Y)$, this implies $\history(X)\cap\history(Z)=\nothing$ and thus $X\orthF Z$ as claimed.

  ``$\Leftarrow$'':
  Suppose that \(Y \orthF Z \impl X\orthF Z\) holds for all variables \(Z\) on \(\Omg\). We claim that $\history(X)\subseteq\history(Y)$ holds.
  To this end, let $i\in \history(X)$ be arbitrary. We choose $Z=U_i$, which implies $\history(Z)=\{i\}$.
  Since $\history(X)\cap\history(Z)=\{i\}\neq\nothing$, we have $X\not\orthF Z$. By applying the contrapositive  \(X\not\orthF Z \not\impl Y \not\orthF Z \) of the assumption, this implies $Y\not\orthF Z$, which implies $\history(Y)\cap\history(Z)\neq\nothing$ and thus $i\in\history(Y)$.
  Since this holds for all $i\in\history(X)$, we have $\history(X)\subseteq\history(Y)$ as claimed, and so $X$ is structurally before $Y$.
\end{proof}

\section{Proofs for Section \ref{sec:dag}}
\subsection{Compositional semigraphoid}
\label{sec:proofcomposition}

Here we prove \Cref{prop:compositionalsemigraphoid}, which states that structural independence is a compositional semigraphoid. We first establish the composition axiom.
\begin{lemma}[Composition axiom]
\label{lemma:composition}
  Let $X$, $Y$, $Z$, and $W$ be random variables in a factored space $\FS=\timesOmg$.
  If \(X\orthF Y \given W\) and \(X\orthF Z\given W\) hold, then \(X\orthF (Y,Z)\given W\) holds.
\end{lemma}
\begin{proof}
We start by assuming that the premise $X\orthF Y \given W$ and $X\orthF Z\given W$ holds.
Let $w\in \Val(W)$ be arbitrary.
By \Cref{def:structuralindependence}, we have
\(\history(X\given w)\cap\history(Y\given w)=\nothing \textnormal{ and } \history(X\given w)\cap\history(Z\given w)=\nothing\), which implies
\(\history(X\given w)\cap(\history(Y\given w) \cup\history(Z\given w))=\nothing\).
Together with \Cref{lemma:jointhistory}, we obtain
\(\history(X\given w)\cap\history((Y,Z)\given w) =\nothing\).
As this holds for all $w\in \Val(W)$, we arrive at
\(X\orthF (Y,Z)\given W\),
which is exactly the conclusion of the composition axiom.
\end{proof}

The remaining four axioms can be shown directly in a straightforward manner, but we instead use the soundness and completeness theorem for a short proof.

\compositionalsemigraphoid*
\begin{proof}
We need to show axioms 1--4 and 6 in \Cref{tab:semigraphoid}.
Axiom~6 (composition) is true by \Cref{lemma:composition}.
Moreover, Axioms 1--4 follow directly from \Cref{thrm:soundcomplete} and the fact that statistical independence is a semigraphoid \cite{pearl1988probabilistic}.
\end{proof}

\subsection{Distribution over DAG and corresponding FSM}
\label{sec:proofdistributionconstruction}

Here we provide a proof for \Cref{lemma:taubijective} which  establishes the relationship between a DAG $G=(V,E,\Val)$ and the FSM $\FSMG=(\Omg, \Xb)$ constructed from $G$. 

First, we prove a preparatory lemma.
\begin{lemma}
\label{lemma:pomg_preparatory}
Let $G=(V,E,\Val)$ be a DAG, let $\FSMG=(\Omg, \Xb)$ be the factored space model constructed from $G$. Let $\Pomg$ factorize over $\Omg$.
Then we have
\[\Pomg(X=x)=\prod_\vinV\Pomg(U_\vxpav=x_v)\,.\]
\end{lemma}

\begin{proof}
We claim that for any fixed $x \in \Val$, we have
$\{X=x\} = \{\omega \in \Omega: \forall v \in V: U_{v, \xpav} (\omega) = x_v \}$. We show the inclusion in both directions.

'$\subseteq$':
Let $\omega \in \Omega$, $\vinV$, and let $X(\omega) = x$. Recall that $X$ is defined in~\eqref{eq:def-graph-X} as $X_v(\omg)\coloneqq \omg_{(v, \Xpav(\omg))}$, so we have that $x_v = X_v (\omega) = U_{v,\xpav} (\omega)$.
Therefore, $\omega \in \{\omega \in \Omega: \forall v \in V: U_{v, \xpav}(\omg) = x_v \}$.

'$\supseteq$':
Let $\omega \in \Omega$ such that $\forall v \in V: U_{v, \xpav}(\omg) = x_v$.
We show that $X_v (\omega)=x_v$ for all $v \in V$ by structural induction over $G$. If $v$ is a root node, then $\pa(v)=\nothing$, and by \eqref{eq:def-graph-X} we have $X_v(\omg)=\omg_{(v,())}=U_{v,()}(\omg)=x_v$. Induction step: Let $v \in V$ be arbitrary and $X_{v'} (\omega)=x_{v'}$ for all ${v'} \in \pa(v)$.
Then, $X_v (\omega) = U_{v,\Xpav(\omega)}(\omega) = U_\vxpav (\omega) = x_v$.

We conclude
$
P^\Omega (X=x)
= P^\Omega (\forall v \in V : U_{v,\xpav} = x_v)
= \prod_{v \in V} P^\Omega (U_{v,\xpav} = x_v)
$.
\end{proof}

Now we prove \Cref{lemma:taubijective}, which we restate here for convenience.

\taubijective*

\begin{proof}
  First we must establish that $\tau$ is a function from $\distributionsF(\Omg)\to \Delta^*(G)$.
  Let $P^\Omg\in \distributionsF(\Omg)$ and $P\coloneqq\tau(P^\Omg)$, we claim that $P$ factorizes over~$G$.
  Indeed, suppose for convenience that $1,2,\dots,n\in V$ is a topological ordering of the vertices of~$G$.
  Then for all $x\in\Val$, we have the following:
  \begin{align*}
    P(x)
    &=P^\Omg(X=x)=P^\Omg(\bigwedge_{v=1}^n X_v=x_v)=\prod_{v=1}^n P^\Omg(X_v=x_v\given X_1=x_1 \wedge\dots\wedge X_{v-1}=x_{v-1})\,.\\
  \intertext{Here, the last equality follows from the definition of conditional probability.
  Recall from \eqref{eq:def-graph-X} that~$X$ is defined via $X_v(\omg)=\omg_{(v,\Xpav(\omg))}$. Since $1,\dots,n$ is a topological ordering of $V$, by conditioning on $X_1=x_1\wedge\dots\wedge X_{v-1}=x_{v-1}$, we have $X_v(\omg)=\omg_{(v,\xpav)}$. Thus, we can continue the calculation as follows:}
    &=\prod_{v=1}^n P^\Omg(X_v=x_v\given X_{\pav}=x_{\pa(v)})=\prod_{v\in V} P(x_v\given x_{\pav})\,.
  \end{align*}
  The first equality follows by inductively applying~\eqref{eq:def-graph-X}: $X_v$ is a deterministic function of $X_{\pa(v)}$ and the values $\omg_{(v,\xpav)}$ alone. Therefore, by induction on~$u\in V= \{1,\dots,n\}$, for all $u\in V$, we have that $X_u$ is a function of the entries $\omega_{(u',x_\pa(u')}$ for $u
'\leq u$. Under the conditioning $X_{\pa(v)}=x_{\pa(v)}$, this shows that $X_v$ is conditionally independent in $\Pomg$ from all $X_u$ with $u<v$ and $u\not\in\pa(v)$. The ultimate equality follows from the definition of $P$. The equation shows that $P$ factorizes over $G$, and thus $P\in \distributions^*(G)$.

  It remains to argue that $\tau$ is bijective.
  To do so, it suffices to show that the given function $\tau^{-1}$ is a function from $\Delta^*(G)\to\distributionsF(\Omg)$ and that $\tau^{-1}$ is indeed the inverse of $\tau$.
  We first observe that the given function $\tau^{-1}$ is a function from $\Delta^*(G)\to\distributionsF(\Omg)$, because $\tau^{-1}(P)$ is by definition a product of distributions over $\Omg_i$ for all $i\in I$.
  Finally, we must show that $\tau^{-1}$ is the inverse of $\tau$, so let $P^\Omg\in\distributionsF(\Omg)$ and $P \in\distributions^*(G)$.
  We claim that $\tau^{-1}(\tau(\Pomg))=P^\Omg$ and $\tau(\tau^{-1}(P)=P$ holds.
  Indeed, for all $\omg\in\Omg$, we have the following:
  \begin{align*}
    \tau^{-1}(\tau(\Pomg))(\omg)
    &=
    \prod_{(v,x_{\pa(v)})\in I} \tau(\Pomg)(\omg_{v,x_{\pa(v)}}\given x_{\pa(v)})
    &\text{(by definition of $\tau^{-1}$)}
    \\
    &=
    \prod_{(v,x_{\pa(v)})\in I} P^{\Omg}(X_v=\omg_{v,x_{\pa(v)}}\given X_{\pa(v)}=x_{\pa(v)})
    &\text{(by definition of $\tau$)}
    \\
    &=
    \prod_{(v,x_{\pa(v)})\in I} P^{\Omg}(U_{v,X_{\pa(v)}(\omg)}=\omg_{v,x_{\pa(v)}}\given X_{\pa(v)}=x_{\pa(v)})
    &\text{(by definition of $X_v$)}
    \\
    &=
    \prod_{(v,x_{\pa(v)})\in I} P^{\Omg}(U_{v,x_{\pa(v)}(\omg)}=\omg_{v,x_{\pa(v)}})
    &\text{(by independence)}
    \\
    &=\prod_{(v,x_{\pa(v)})\in I} P^{\Omg}(\omg_{v,x_{\pa(v)}})=P^\Omg(\omg)\,.
    &\text{(since $P^\Omg$ factorizes)}
  \end{align*}
Further, we have
  \begin{align*}
    \tau(\tau^{-1}(P))(x)
    &=\tau^{-1}(P)(X=x)
    &\text{(by definition of $\tau$)}
    \\
    &= \prod_\vinV \tau^{-1}(P)(U_\vxpav=x_{v})&\text{(by \Cref{lemma:pomg_preparatory})}\\ 
    &= \prod_\vinV \tau^{-1}(P)(U_\vxpav=x_{v} \given X_\pav = \xpav) &\text{(by independence)}\\ 
    &= \prod_\vinV \tau^{-1}(P)(X_{v} =x_{v} \given X_\pav = \xpav) &\text{(by definition)}\\ 
    &= \prod_\vinV P(x_v\given \xpav)&\text{(by definition of $\tau^{-1}$)}\\
    &= P(x) &\text{(since $P$ factorizes)}
    \end{align*}
  Thus, $\tau$ is a bijective function with inverse $\tau^{-1}$. This concludes the proof. 
\end{proof}

\subsection{Equivalence of d-separation and structural independence for nodes}
\label{sec:proof-d-sep-str-indep}

Here we prove \Cref{prop:dseparation-structuralindependence}, which states that d-separation for nodes in a graph $G$ holds if and only if the corresponding variables are structurally independent in the FSM $\FSMG$ constructed from $G$. This proof is an application of \Cref{thrm:soundcomplete}, but it is also possible to prove \Cref{prop:dseparation-structuralindependence} directly without \Cref{thrm:soundcomplete}.

\dseparationstructuralindependence*
\begin{proof}
    By \Cref{lemma:taubijective}, for any distribution $P$ which factorizes over $G$, there is a $\Pomg$ over $\Omg$ with $P(x)=\Pomg(\Xb=x)$.
\Cref{lemma:taubijective} also implies that for any distribution $\Pomg$ over $\Omg$ that factorizes over $\FSMG$ there is a distribution $P$ which factorizes over $G$ with $P(x)=\Pomg(\Xb=x)$.
Therefore, with \Cref{def:conditionalindependence}, it follows that for  all values $x_{V_1}, x_{V_2}, x_{V_3}\in \Val_{V_1}\times \Val_{V_1}\times \Val_{V_3}$
\begin{equation}
\begin{aligned}
\label{eq:distequivalence}
    &\text{$x_{V_1} \indep^{P} x_{V_2} \given x_{V_3}$
     for all distributions $P$ that factorize over $G$}\\
     &\hspace{4cm}\iff\\
    &\text{$x_{V_1} \indep^{\Pomg} x_{V_2} \given x_{V_3}$  for all distributions $\Pomg$ that factorize over $\Omg$.}
\end{aligned}
\end{equation}
We have that
\begin{align*}
&\text{$V_1$~and~$V_2$ are d-separated given $V_3$ in $G$.}\\
\overset{(i)}{\iff}&\text{$X_{V_1} \indep^{P} X_{V_2} \given X_{V_3}$ for all distributions $P$ that factorize over $G$.} \\
\overset{(ii)}{\iff} &\text{$X_{V_1} \indep^{\Pomg} X_{V_2} \given X_{V_3}$ for all distributions $\Pomg$ that factorize over $\FSMG$.}\\
\overset{(iii)}{\iff} &\text{$X_{V_1} \orthF X_{V_2} \given X_{V_3}$.}
\end{align*}
Wherein $(i)$ is true by the soundness and completeness of d-separation \cite{koller2009probabilistic}, $(ii)$ follows from \eqref{eq:distequivalence}, and $(iii)$ is true by  \Cref{thrm:soundcomplete}.
This concludes the proof.
\end{proof}
\subsection{Perfect maps of DAGs and factored spaces}\label{sec:proofperfectmap}
Here we prove \Cref{prop:perfectmap}, which states that if there is a DAG that is a perfect map of a distribution $P$, then there is a factored space model that is also a perfect map of $P$, but not vice versa.

\begin{proof}[Proof of \Cref{prop:perfectmap}]
    To prove 1., let $G$ be a DAG that is a perfect map of $P$. We want to prove that $\FSMG=(\Omg, X)$ is a perfect map of $P$ with regard to $X$. Let $X_{V_1},X_{V_2}, X_{V_3} \subseteq X$ be arbitrary. Then, by \Cref{prop:dseparation-structuralindependence}, $X_{V_1}$ and $X_{V_2}$ are structurally independent given $X_{V_3}$ in~$\FS$ if and only if they are d-separated in $G$. Since $G$ is a perfect map, this is equivalent to $X_{V_1}$ and $X_{V_2}$ being independent given $X_{V_3}$.

    To prove 2., consider the following counterexample. Let $\Omg=\Omg_1=\{0, 1, 2\}$, let $P$ be an arbitrary distribution over $\Omg$, and let $X=(X_1,X_2)$ with $X_1=U_1$, and $X_2= [U_1>0]$. Since $X_2$ is a deterministic function of $X_1$, we have the following independence and dependence relations:\\

\begin{tabular}{@{}lll@{}}
nontrivial & trivial independencies & trivial dependencies\\
         $X_2\indep X_2\given X_1$&$X_1\indep X_1\given X_1$ & $X_1\not\indep X_1$ \\
         $X_1\not\indep X_1\given X_2$&$X_2\indep X_2\given X_2$ &$X_2\not\indep X_2$ \\
         $X_1\not\indep X_2$&$X_1\indep X_1\given \{X_1, X_2\}$ &$X_1\not\indep \{X_1,X_2\}$\\
         &$X_2\indep X_2\given \{X_1, X_2\}$  & $X_2\not\indep \{X_1,X_2\}\,,$\\
    \end{tabular}\\\\
    as well as the independencies which follow directly from the listed independencies.
    To see that $\FSModel=(\Omg,X)$ is a perfect map of $P$, note that the histories are $\history(X_1)=\history(X_2)=\{1\}$, $\history(X_1\given X_2)=\{1\}$, and $\history(X_2\given X_1)=\nothing$. One an easily verify that the above independencies correspond to the structural independencies. In particular $X_2\orthF X_2\given X_1$. However, there is no graph with nodes $\{v_1,v_2\}$ such that $v_2$ is d-separated from itself given $v_1$, since the path from $v_2$ to itself has zero edges and can only be blocked by itself.
\end{proof}
For a more complex example when there is no graph that is a perfect map for a distribution, see \Cref{sec:background}.

\section{Proofs for Section \ref{sec:soundcomplete}} 
\label{app:soundcompleteproof}

Throughout this section, let $\Omg=\timesOmg$ be a factored space. Let us state some standard definitions.
\begin{definition}[Outer product]\label{def:outer-product} Let $J,K\subseteq I$ be disjoint. Let $P_J$ and $P_K$ be distributions over $\Omg_J$, and $\Omg_K$, respectively. Then the \emph{outer product} of $P_J$ and $P_K$ is the distribution $P_J\otimes P_K$ over $\Omg_{J\cup K}$ with $(P_J\otimes P_K)(\alpha)\coloneq P_J(\alpha_J)P_K(\alpha_K)$ for all $\alpha\in\Omg_{J\cup K}$.
\end{definition}
\begin{definition}[Marginal distribution]\label{def:marginal}
Let $P$ be an arbitrary distribution over $\Omg$, and let $J\subseteq I$.
Then the \emph{marginal distribution} of $P$ in $J$ is given by $P_J = P \compose U_J^{-1}$.
For $\iinI$, we write $P_i\coloneq P_{\{i\}}$.
\end{definition}

Note that if $P$ factorizes over $\Omg$, then it is the outer product of its marginal distributions in $i\in I$. That is, $P=\bigotimes_\iinI P_i$. In that case, the marginal distribution is $P_J=\bigotimes_{i\in J}P_i$.
For a finite set~$S$, we also write $\distributions(S)$ for the set of all distributions on~$S$.
Then the set $\distributionsF(\Omg)$ of all distributions that factorize over $\Omg$ satisfies $\distributionsF(\Omg)= \{\bigotimes_\iinI  P_i \mid \forall i\in I\colon P_i \in \distributions(\Omg_i)\}$.

\subsection{Proof of soundness for events}
\label{sec:soundevent}
In this section, we prove \Cref{lemma:soundevent}, the soundness of structural independence for events.
We first characterize derived variables (\Cref{def:derived}) in an alternative way as follows.
\begin{lemma}
  \label{lemma:characterization_derived_variables}
  Let $X \colon \Omg \to \Val(X)$ and $Y \colon \Omg \to \Val(Y)$ be two random variables and $C \subseteq \Omg$.
  Then the following statements are equivalent:
  \begin{enumerate}[label=(\roman*),leftmargin=*]
     \item\label{item:derived_1} $X \functionTo_C Y$.
     \item\label{item:derived_2} For all $\omg, \omg' \in C$, we have that $X(\omg) = X(\omg')$ implies $Y(\omg) = Y(\omg')$.
  \end{enumerate}
\end{lemma}
\begin{proof}
  \ref{item:derived_1} $\implies$ \ref{item:derived_2}:
  By $X \functionTo_C Y$, there is a function $f \colon \Val(X) \to \Val(Y)$ with $Y(\omg) = f(X(\omg))$ for all $\omg \in C$.
  For all $\omg, \omg' \in C$,
  if $X(\omg) = X(\omg')$, then $Y(\omg) = f(X(\omg)) = f(X(\omg')) = Y(\omg')$, so \ref{item:derived_2} holds.

  \ref{item:derived_2} $\implies$ \ref{item:derived_1}:
  We construct $f \colon \Val(X) \to \Val(Y)$ as follows:
  For $x \in \Val(X)$, if there exists $\omg \in C$ with $X(\omg) = x$, define $f(x) \coloneqq Y(\omg)$.
  By \ref{item:derived_2}, $f(x)$ does not depend on the choice of $\omg \in C$ with $X(\omg) = x$.
  Let $f(x)$ be arbitrary if there is no $\omg \in C$ with $X(\omg) = x$.
  By construction, we conclude $Y(\omg) = f(X(\omg))$ for all $\omg \in C$, and thus \ref{item:derived_1} holds.
\end{proof}

Recall from \Cref{def:history} that $J$ generates $X$ given $C$ if $J$ disintegrates $C$ and $U_J \functionTo_C X$ holds.
We naturally extend this definition to events~$A\subseteq\Omg$ in that $J$ generates $A$ given~$C$ if $J$ generates $1_A$ given~$C$, where $1_A$ is the indicator random variable that is $1$ for all $\omg\in A$ and $0$ otherwise.
In the following lemma, we characterize \Cref{def:history} differently.

\begin{lemma}[Alternative characterization of generation]
  \label{lemma:characterization_event_generation}
  Let $\FS = \timesOmg$ be a factored space, let $A, C \subseteq \Omg$, let $J \subseteq I$ be a set with complement $\comp{J} = I \setminus J$, and suppose that $J$ disintegrates $C$.
  The following statements are equivalent:
  \begin{enumerate}[label=(\roman*),leftmargin=*]
    \item\label{item:gen-1} $J$ generates $A$ given $C$.
    \item\label{item:gen-2} For all $\omg, \omg' \in C$ with $\omg_J = \omg_J'$, we have $\omg \in A$ if and only if $\omg' \in A$.
    \item\label{item:gen-3} $A \cap C = (A \cap C)_J \times C_{\comp{J}}$.
  \end{enumerate}
  Moreover, the history $\history{A \given C}$ is the smallest set $J$ that disintegrates $C$ and that satisfies $A \cap C = (A \cap C)_J \times C_{\comp{J}}$.
\end{lemma}
\begin{proof}
  \ref{item:gen-1} $\implies$ \ref{item:gen-2}:
  Let $\omg, \omg' \in C$ with $\omg_J = \omg_J'$.
  This means $U_J(\omg) = U_J(\omg')$, and by \ref{item:gen-1} we have $U_J \functionTo_C 1_A$.
  By Lemma~\ref{lemma:characterization_derived_variables}, we obtain $1_A(\omg) = 1_A(\omg')$. Thus, \ref{item:gen-2} holds.

  \ref{item:gen-2} $\implies$ \ref{item:gen-3}:
  The inclusion $A \cap C \subseteq (A \cap C)_J \times C_{\comp{J}}$ holds by definition of the Cartesian product.
  For the other inclusion, we have
  \(
    (A \cap C)_J \times C_{\comp{J}} \subseteq C_J \times C_{\comp{J}} = C
  \)
  since $J$ disintegrates $C$.
  Thus, it remains to show that $(A \cap C)_J \times C_{\comp{J}} \subseteq A$ holds.
  Let $\omg \in (A \cap C)_J \times C_{\comp{J}}$.
  Then $\omg = b_J \merge c_{\comp{J}}$ with $b \in A \cap C$ and $c \in C$.
  We have $b, \omg \in C$ and $b_J = \omg_J$, which by \ref{item:gen-2} implies that $b \in A$ holds if and only if $\omg \in A$ holds.
  Since $b \in A$, we conclude $\omg \in A$, which was to be shown. This proves \ref{item:gen-3}.

  \ref{item:gen-3} $\implies$ \ref{item:gen-1}:
  Let $\omg, \omg' \in C$ and assume $U_J(\omg) = U_J(\omg')$.
  In order to establish \ref{item:gen-1}, it suffices to show that $1_A(\omg) = 1_A(\omg')$ holds, since Lemma~\ref{lemma:characterization_derived_variables} then implies $U_J \functionTo_C 1_A$ as required.
  For reasons of symmetry, it is enough to show that $\omg \in A$ implies $\omg' \in A$.
  Assume $\omg \in A$, so $\omg \in A \cap C$.
  We obtain
  \begin{equation*}
    \omg' = \omg'_J \merge \omg'_{\comp{J}} = \omg_J \merge \omg'_{\comp{J}} \in (A \cap C)_J \times C_{\comp{J}} = A \cap C,
  \end{equation*}
  wherein the last step uses \ref{item:gen-3}.
  This shows $\omg' \in A$.
  Thus, \ref{item:gen-1} holds.

  Finally, recall from~\Cref{lemma:historygenerates} that the history $\history(A \given C)$ is the smallest set $J$ that generates $A$ given $C$. Together with \ref{item:gen-1} $\Longrightarrow$ \ref{item:gen-3}, this proves the final claim.
\end{proof}

We are now ready to prove the soundness of structural independence for events.
\soundevent*
\begin{proof}
Suppose that $\history(A\given C)\cap \history(B\given C)=\nothing$ holds.
Let $P \in \distributionsF{\factors}$ be an arbitrary distribution that factorizes over~$\Omg$. We need to prove $A\indep^PB\given C$.

Let $J \coloneqq \history{A \given C}$.
By $\history(A\given C)\cap \history(B\given C)=\nothing$, we have $\history{B \given C} \subseteq \comp{J}=I\setminus J$.
Since $J$ generates $A$ given $C$ by \Cref{lemma:historygenerates}, \Cref{lemma:characterization_event_generation} implies
  \begin{equation*}
    A \cap C = (A \cap C)_J \times C_{\comp{J}}.
  \end{equation*}

  We now claim that $\comp{J}$ generates $B$ given $C$.
  First, $\comp{J}$ disintegrates $C$ since $J$ disintegrates $C$.
  Since $\history{B \given C} \subseteq \comp{J}$ holds and $\history{B \given C}$ generates $B$ given $C$, we obtain
  \(
    U_{\comp{J}} \functionTo_C U_{\history{B \given C}} \functionTo_C 1_B
  \), and so $\comp{J}$ generates $B$ given $C$.
  Thus, \Cref{lemma:characterization_event_generation} also implies
  \begin{equation*}
    B \cap C = (B \cap C)_{\comp{J}} \times C_{J} = C_J \times (B \cap C)_{\comp{J}}\,,
  \end{equation*}
  wherein we use commutativity of the Cartesian product over indexed families in the last step.
  Moreover, we observe the following identity:
  \begin{equation*}
  \begin{aligned}
    A \cap B \cap C &= (A \cap C) \cap (B \cap C)\\ &= ((A\cap C)_J\times C_\comp{J})\cap (C_J\times (B\cap C)_\comp{J}))\\
    &=((A\cap C)_J\cap C_J) \times (C_\comp{J}\cap(B\cap C)_\comp{J}))\\
    &= (A \cap C)_J \times (B \cap C)_{\comp{J}}.
  \end{aligned}
  \end{equation*}
  Finally, using that $P$ factorizes over the factors and the previous identities, we obtain
  \begin{align*}
    P(A \cap C) \cdot P(B \cap C) &= P\big( (A \cap C)_J \times C_{\comp{J}} \big) \cdot P \big( C_J \times (B \cap C)_{\comp{J}} \big) \\
    &= P_J\big( (A \cap C)_J \big) \cdot P_{\comp{J}}\big( C_{\comp{J}} \big) \cdot P_{J}(C_J) \cdot P_{\comp{J}}\big( (B \cap C)_{\comp{J}} \big) \\
    &= P\big( (A \cap C)_J \times (B \cap C)_{\comp{J}}\big) \cdot P\big( C_J \times C_{\comp{J}}  \big) \\
    &= P(A \cap B \cap C) \cdot P(C).
  \end{align*}
  This is exactly the definition of $A \indep^P B \given C$.
\end{proof}

\subsection{Local-to-global principle}
\label{sec:local_to_global}
Here we prove the local-to-global principle (\Cref{lemma:openset}). We first prove a simple preparatory lemma.

\begin{lemma}
  \label{lemma:obtain_polynomial}
  Let $\FS = \timesOmg$ be a factored space and $Q, P \in \distributionsF(\Omega)$.
  For $\lambda \in [0, 1]$, define $R^{\lambda}_i = (1 - \lambda) Q_i + \lambda P_i$ and $R^{\lambda} = \bigotimes_{i \in I} R^{\lambda}_i$.
  Let $A \subseteq \Omg$ be an event.
  Then the function $\lambda \mapsto R^{\lambda}(A)$ is a polynomial of degree at most $|I|$ in $\lambda$.
\end{lemma}

\begin{proof}
  We have
  \begin{align*}
    R^\lambda(A) &= \sum_{a \in A} R^{\lambda}(a)
    = \sum_{a \in A} \prod_{\iinI} R^{\lambda}_{i}(a_i)
    =  \sum_{a\in A}\prod_\iinI \Big( (1-\lambda)Q_i(a_i) + \lambda P_i(a_i) \Big).
  \end{align*}
  Thus, $R^\lambda(A)$ is a polynomial of degree at most $|I|$.
\end{proof}

Now we prove \Cref{lemma:openset}, the local-to-global principle, which states that if there is an $\epsilon$-ball $B$ of distributions $Q'$ such that $X\indep^{Q'}Y\given Z$ for all $Q'\in B$, then $X\indep^{P}Y\given Z$ for all $P\in\distributionsF(\Omg)$.

\begin{proof}[Proof of \Cref{lemma:openset}]
  Let $B\subseteq \distributionsF(\Omg)$ be the set of all $Q'\in\distributionsF(\Omg)$ with $d(Q, Q') < \epsilon$.
    Let $\epsilon$ be such that $X \indep^{Q'} Y \given Z$ for all $Q'\in B$. Let $P\in \distributionsF{\factors}$, $x\in\Val(X)$, $y\in\Val(Y)$, and $z\in\Val(Z)$ be arbitrary.
  We need to show $x \indep^{P} y \given z$, or equivalently $P(x, z)P(y, z) = P(x, y, z) P(z)$.

    With the parameter $\lambda\in[0,1]$, we now define a distribution $R^\lambda$ that interpolates between $Q$ and $P$ as follows:
    For all $i\in I$, we define $R^\lambda_i=(1-\lambda) Q_i+\lambda P_i$, and we let $R^\lambda=\bigotimes_\iinI  R^\lambda_i$.
    Then we have $R^0=Q$ and $R^1=P$.
    Note that the map $\lambda \mapsto R^{\lambda}$ is continuous.
    Thus, $R^\lambda\to Q$ as $\lambda\to 0^+$.
    As $B$ is an open set around~$Q$ and $R^{\lambda} \to 0$ as $\lambda \to 0^+$, there are then infinitely many $\lambda>0$ with $R^\lambda\in B$.
 
    Now, consider the function
    \begin{equation*}
      f(\lambda)=R^{\lambda} (x, z) R^{\lambda} (y, z) - R^{\lambda} (x, y , z) R^{\lambda} (z),
    \end{equation*}
    which by Lemma~\ref{lemma:obtain_polynomial} is a polynomial of degree at most $2 |I|$.
    For infinitely many $\lambda>0$, we have $R^\lambda\in B$, and thus $f(\lambda)=0$ by construction of $B$. Since $f$ is a polynomial with infinitely many roots, it must be identically zero.
    Thus
    \begin{equation*}
      P(x, z)P(y , z) - P(x , y , z) P(z) = f(1) = 0,
    \end{equation*}
    that is, $x \indep^{P} y \given z$.
    This concludes the proof.
\end{proof}

\subsection{Proof of completeness for events}
\label{sec:completeevents}
Here we prove the completeness of structural independence for events (\Cref{lemma:completeevent}).
We introduce further notation for this section:
\begin{itemize}
\item $\indepstar$: For $A, B, C \subseteq \Omg$, we write $A \indepstar B \given C$ if  $A \indep^P B \given C$ holds for all $P \in \distributionsF(\Omg)$.
  \item $\distributionsF_C(\Omg)$: For an event $C \subseteq \Omg$, let
  $\distributionsF_C(\Omg) \coloneqq \big\lbrace P \in \distributionsF(\Omg): P(C) > 0 \big\rbrace$. That is, the set of distributions~$P$ that factorize over $\FS$ and assign a positive probability to $C$.
  \item $\distributionsF_{C,i} (\Omg)$: For $i \in I$, let $\distributionsF_{C, i}(\Omg) \coloneqq
    \Big\lbrace(P,Q): P,Q \in \distributionsF_{C}(\Omg) \text{ and } P_{j} = Q_{j} \text{ for all $j\in I\setminus\{i\}$}\Big\rbrace$   denote the pairs of distributions in $\distributionsF_C(\Omg)$ that can only differ in $i$.
\end{itemize}
We aim to show that $A \indepstar B \given C$ implies $\history{A \given \cvar} \sect \history{B \given \cvar} = \nothing$.

A key ingredient to the proof is the concept of \emph{irrelevant} factors, which form the \emph{cohistory}.

\begin{definition}[irrelevant factors, cohistory]
  Let $A,\cvar \subseteq \factors$, and let $(P,Q) \in \distributionsF_{C, i}(\Omg)$.
  We say that $i\in I$ is \emph{$(P,Q)$-irrelevant} to $A$ given $C$ if $P(A \given \cvar) = Q(A \given \cvar)$. A factor $i \in I$ is (globally) \emph{irrelevant} to $A$ given $C$, if it is $(P,Q)$-irrelevant for all $(P,Q) \in \distributionsF_{C, i}(\Omg)$.
  The \emph{cohistory} of $A$ given $C$, denoted $\historyC(A \given \cvar)$ is the set of all $i\in I$ that are irrelevant to $A$ given $C$.
\end{definition}
We say $i$ is \emph{relevant} if it is not irrelevant, and \emph{$(P,Q)$-relevant} if it is not $(P,Q)$-irrelevant to $A$ given~$C$.
Intuitively, $i$ is relevant to $A$ given $C$, when changing $P$ only in the factor $i$ can change $P(A\given C)$.
Note that global irrelevance of $i$ implies the $(P,Q)$-irrelevance of $i$ for all $(P,Q)\in\distributionsF_{C,i}$, while global relevance only implies $(P,Q)$-relevance for at least one pair $(P,Q)\in \distributionsF_{C,i}$.

Recall that for a set $J \subseteq I$, we write $\comp{J} = I \setminus J$ for the complement of $J$ in $I$.
We will show the following two lemmas, which combined result in Lemma~\ref{lemma:completeevent}:

\begin{lemma}
    \label{lemma:completeness_cohistory}
    Let $A,B,\cvar \subseteq \factors$ with $A \indepstar B \given C$.
    Then $\historyC(A \given \cvar) \union \historyC(B \given \cvar) = I$.
\end{lemma}

\begin{lemma}
    \label{lemma:cohistory_is_complement}
    Let $A,C \subseteq \factors$.
    We have $\historyC(A \given \cvar)=\comp{\history{A \given \cvar}}$.
\end{lemma}

\begin{proof}[Proof of \Cref{lemma:completeevent}]
\label{proof:completeevent2}
  \Cref{lemma:completeness_cohistory} states that $A \indepstar B \given C$ implies $\historyC(A \given \cvar) \union \historyC(B \given \cvar) = I$.
  By \Cref{lemma:cohistory_is_complement},
  this is equivalent to $\history{A \given \cvar} \sect \history{B \given \cvar} = \nothing$.
\end{proof}

Thus, our task reduces to proving~\Cref{lemma:completeness_cohistory,lemma:cohistory_is_complement}, to which we devote two separate subsections.

\subsubsection{Proof of Lemma \ref{lemma:completeness_cohistory}}

Let $A, B, C \subseteq \Omg$.
and assume $A \indepstar B \given C$.
We want to show that $\historyC(A \given C) \cup \historyC(B \given C) = I$, which means that for all $i \in I$, $i$ is globally irrelevant to one of $A$ or $B$ given $C$. 
Before we show this, we show a weaker claim, the mutual exclusion principle. It states that for any given pair of distributions $(P,Q)\in\distributionsF_{C, i}(\Omg)$,  $i$ cannot be $(P,Q)$-relevant to both $A$ given $C$ and $B$ given $C$. However, the mutual exclusion principle makes no claim that $i$ is globally irrelevant to $A$ given $C$ or to $B$ given $C$.
We later strengthen the mutual exclusion principle to show the claim about global irrelevance.

\begin{lemma}[Mutual exclusion principle]
  \label{lemma:exclusion_principle}
  Let $A, B, C \subseteq \Omg$ with $A \indepstar B \given C$.
  Let $i \in I$, and let $(P, Q) \in \distributionsF_{C, i}(\Omg)$. Then $i$ cannot both be $(P,Q)$-relevant to $A$ given $C$ and $(P,Q)$-relevant to $B$ given $C$.
\end{lemma}

\begin{proof}
  We define the distribution $R$ as $R(\omg) \coloneqq (P(\omg) + Q(\omg)/2$.
  Since for $j\ne i$, $P$ and $Q$ have identical factors, we have $R_j=P_j=Q_j$. Further, $R$ has the factor $R_i$ with $R_i(\omg_i) = (P_i(\omg_i) + Q_i(\omg_i))/2$ in $i$.
  Consequently, $R \in \distributionsF(\Omg)$.
  By our assumption, $A \indep^R B \given C$ holds.
  We will show that this implies that $P(A \given C) = Q(A \given C)$ or $P(B \given C) = Q(B \given C)$ holds.

  The independence $A \indep^R B \given C$ means that
  \begin{equation*}
    R(A \cap C) \cdot R(B \cap C) = R(A \cap B \cap C) \cdot R(C).
  \end{equation*}
  By substituting the definition of $R$ and multiplying both sides by $4$, we obtain
  \begin{equation*}
    \big( P(A \cap C) + Q(A \cap C) \big) \cdot \big(P(B \cap C) + Q(B \cap C) \big)
    =
    \big( P(A \cap B \cap C) + Q(A \cap B \cap C) \big) \cdot \big( P(C) + Q(C) \big).
  \end{equation*}
  The left side of the equation multiplies out to
    \begin{equation*}
    P(A \cap C) \cdot P(B \cap C)+ Q(A \cap C) \cdot Q(B \cap C)+ P(A \cap C) \cdot Q(B \cap C) + Q(A \cap C) \cdot P(B \cap C)
  \end{equation*}
  and the right side multiplies out to
    \begin{equation*}
P(A \cap B \cap C) \cdot P(C)+Q(A \cap B \cap C) \cdot Q(C)+ P(A \cap B \cap C) \cdot Q(C) + Q(A \cap B \cap C) \cdot P(C).
  \end{equation*}
  As $A \indep^P B \given C$ and $A \indep^Q B \given C$, the first two terms on each side cancel out, and we have
  \begin{equation*}
    P(A \cap C) \cdot Q(B \cap C) + Q(A \cap C) \cdot P(B \cap C) = P(A \cap B \cap C) \cdot Q(C) + Q(A \cap B \cap C) \cdot P(C).
  \end{equation*}
  By definition of $\distributionsF_C(\Omg)$, $P(C)>0$ and $Q(C)>0$, and we can divide both sides of the equation by $P(C) \cdot Q(C)$.
  \begin{align*}
    P(A \given C) \cdot Q(B \given C) + Q(A \given C) \cdot P(B \given C)
    &= P(A \cap B \given C) + Q(A \cap B \given C) \\
    &= P(A \given C) \cdot P(B \given C) + Q(A \given C) \cdot Q(B \given C),
  \end{align*}
  where we transformed the right-hand side further by again using the independencies for $P$ and $Q$.
  By subtracting the right-hand side from the equation, and then factoring the left-hand side of the result, we obtain
  \begin{equation*}
    \big( P(A \given C) - Q(A \given C) \big) \cdot \big( Q(B \given C) - P(B \given C) \big) = 0.
  \end{equation*}
  Thus, at least one of the factors needs to be zero, which means $P(A \given C) = Q(A \given C)$ or $P(B \given C) = Q(B \given C)$.
\end{proof}

To strengthen the mutual exclusion principle to global irrelevance via an interpolation argument, we first prove the following simple lemma.

\begin{lemma}
  \label{lemma:C_positivity_preserved}
  Let $P, P' \in \distributionsF_{C}(\Omg)$.
  For $\lambda \in [0, 1]$, define $P^{\lambda}_j = (1 - \lambda) P_j + \lambda P'_j$ and $P^{\lambda} = \bigotimes_{j \in I} P^{\lambda}_j$.
  Then $P^{\lambda} \in \distributionsF_{C}(\Omg)$.
\end{lemma}

\begin{proof}
    $P^\lambda\in\distributionsF(\omg)$ by definition. It remains to show that $P^\lambda(C)>0$.
  We have
  \begin{align*}
    P^{\lambda}(C) &= \sum_{c \in C} \prod_{j \in I} \Big( (1 - \lambda)P_j(c_j) + \lambda P'_{j}(c_j) \Big)  \\
    & \geq \sum_{c \in C} \Bigg( \prod_{j \in I} (1 - \lambda) P_j(c_j)  + \prod_{j \in I}  \lambda P'_j(c_j) \Bigg)  \\
    &= \sum_{c \in C} (1 - \lambda)^{|I|} P(c) + \sum_{c \in C} \lambda^{|I|}P'(c) \\
    &= (1 - \lambda)^{|I|} P(C) + \lambda^{|I|} P'(C) \\
    &> 0\,,
  \end{align*}
  wherein the first step uses the same expansion as the proof of Lemma~\ref{lemma:obtain_polynomial}.
  For the first inequality, we multiply out the product over $j$, and remove all the summands that contain both $P_j(c_j)$, and $P'_{j'}(c_{j'})$, that is summands of the form $(1-\lambda)P_j(c_j)\cdot\lambda P_{j'}(c_{j'})\cdot \dots$, which are all non-negative.
\end{proof}

Now we prove \Cref{lemma:completeness_cohistory}, which states that $\historyC(A\given C)\cup \historyC(B\given C)=I$.
\begin{proof}[Proof of Lemma~\ref{lemma:completeness_cohistory}]
  Let $i \in I$.
  We now show that $i \in \historyC(A \given C) \cup \historyC(B \given C)$.
  Since the cohistory is the set of irrelevant factors, we need to show that if $i$ is relevant to A given $C$, then  $i$ is irrelevant to $B$ given $C$.

  Assume $i$ is relevant to $A$ given $C$. That means there is a tuple $(P, Q) \in \distributionsF_{C, i}(\Omg)$ with $P(A \given C) \neq Q(A \given C)$, i.e., $P(A \given C) - Q(A \given C) \neq 0$.
  We view $\distributionsF_{C, i}(\Omega) \subset \big(\distributionsF_{C}(\Omg)\big)^2 \subset \mathbb{R}^{2 |\Omega|}$ as a topological space with the induced standard topology.
  Since the function $(P', Q') \mapsto P'(A \given C) - Q'(A \given C)$ is continuous, there is an open set $S \subseteq \distributionsF_{C, i}(\Omg)$ with $(P, Q) \in S$ and with $P'(A \given C) \neq Q'(A \given C)$ for all $(P', Q') \in S$.
  Consequently, the mutual exclusion principle (Lemma~\ref{lemma:exclusion_principle}) guarantees that
  \begin{equation}
  \label{eq:mutual_exclusion_on_U}
      P'(B \given C) = Q'(B \given C)\text{ for all $(P', Q') \in S$.}
  \end{equation}
  In the rest of the proof, we argue that \eqref{eq:mutual_exclusion_on_U} extends to all $(P', Q') \in \distributionsF_{C, i}(\Omg)$, which precisely means that $i \in \historyC(B \given C)$.
  We prove this with a similar strategy to Lemma~\ref{lemma:openset}.

  Let $(P', Q') \in \distributionsF_{C, i}(\Omg)$ be arbitrary. We interpolate between $(P, Q)$ and $(P', Q')$ as follows.
  Let $P^{\lambda}_j \coloneq (1 - \lambda) P_j + \lambda P'_{j}$ and $P^{\lambda} \coloneq \bigotimes_{j \in I} P^{\lambda}_j$, and let $Q^{\lambda}$ be defined analogously.
  Then Lemma~\ref{lemma:C_positivity_preserved} shows that $P^{\lambda}, Q^{\lambda} \in \distributionsF_{C}(\Omg)$.
  One can also easily verify that $P^{\lambda}_j = Q^{\lambda}_j$ for all $j \neq i$, and so $(P^{\lambda}, Q^{\lambda}) \in \distributionsF_{C, i}(\Omg)$.

  Now, consider the function $f: [0, 1] \to \mathbb{R}$ given by
  \begin{equation*}
    f(\lambda) \coloneqq P^{\lambda}(B \cap C) Q^{\lambda}(C) - Q^{\lambda}(B \cap C)P^{\lambda}(C).
  \end{equation*}
  By Lemma~\ref{lemma:obtain_polynomial}, this function is a polynomial in $\lambda$.
  We have $(P^{\lambda}, Q^{\lambda}) \to (P, Q)$ for $\lambda \to 0^{+}$, which means there are infinitely many $\lambda > 0$ for which $(P^{\lambda}, Q^{\lambda}) \in S$.
  For such $\lambda$, we have $P^{\lambda}(B \given C) = Q^{\lambda}(B \given C)$ by \eqref{eq:mutual_exclusion_on_U}. Multiplying with $P^{\lambda}(C) \cdot Q^{\lambda}(C)$, we have that $P^\lambda(B\cap C)Q^\lambda(C)=Q^{\lambda}(B \cap C)P^{\lambda}(C)$. So, for all $\lambda$ with $(P^{\lambda}, Q^{\lambda}) \in S$  we have $f(\lambda) = 0$.
  Thus, the polynomial $f$ has infinitely many roots and must be the zero polynomial.
  Consequently, we have $f(1) = 0$, and thus $P'(B \given C) = Q'(B \given C)$. Since $P'$ and $Q'$ are arbitrary, $i$ is globally irrelevant to $B$ given $C$. This concludes the proof.
\end{proof}

\subsubsection{Proof of Lemma \ref{lemma:cohistory_is_complement}}

Let $A, C \subseteq \Omg$.
In this section, we show that $\historyC(A \given C) = \comp{\history{A \given C}}$. Since $\comp{\historyC}$ is the set of  relevant factors, this means the history contains exactly the relevant factors.
The direction ``$i$ is relevant to $A$ given $C$ $\implies$ $i\in\history(A\given C)$'' is easy, while the direction ``$i\in\history(A\given C)$ $\implies$ $i$ is relevant to $A$ given $C$'' is much more difficult.

\paragraph{Proof sketch for Lemma \ref{lemma:cohistory_is_complement}}
We know that $\history{A \given C}$ is the smallest set that generates $A$ given $C$, so it suffices to prove that the set of relevant factors generates $A$ given $C$.
This involves showing that the set of relevant factors disintegrates $C$,~\Cref{lemma:cohistory_disintegrates}. 
To prove \Cref{lemma:cohistory_disintegrates}, we establish $C$ as the support of a distribution that factorizes into two components.
Establishing that factorization relies on \Cref{lemma:cohistory_has_independence_property}, which states that the background variables of the cohistory are independent of the background variables of its complement given $C$. 
Then, the final proof directly verifies the generation property by using the alternative characterization of generation established in ~\Cref{lemma:characterization_event_generation}.
All other lemmas are preparatory.

\paragraph{Further notation for conditional independence}
In \Cref{def:conditionalindependence}, we introduced the notation $X \indep^P Y \given Z$ for random variables $X, Y, Z$, and $A \indep^P B \given C$ for events $A, B, C$.
In what follows, we also mix these notations.
For example, $X \indep^P Y \given C$ means $x \indep^P y \given C$ for all $x \in \Val(X)$ and $y \in \Val(Y)$, where $x$ and $y$ are treated as the events $X^{-1}(x)$ and $Y^{-1}(y)$, respectively.
Similarly, $X \indep^P B \given C$ means $x \indep^P B \given C$ for all $x \in \Val(X)$.
Finally, we also write $\indepstar$ if the independence holds for all $P \in \distributionsF(\Omg)$, for example $X \indepstar B \given C$.

We start by considering the \emph{support} of a distribution $P$ over $\Omg$, namely the event $\supp(P)\subseteq \Omg$ such that $P(\omg)>0$ for all $\omg\in\supp(P)$. We obtain the following lemma.

\begin{lemma}
  \label{lemma:support_statements}
  Let $J \subseteq I$ and let $P, Q \in \distributions(\Omg)$ be two distributions over $\Omg$, and let $C \subseteq \Omg$.
  Then we have
  \begin{enumerate}
    \item If $\supp(P) \subseteq \supp(Q)$, and $P(C) > 0$, then $Q(C) > 0$.
    \item If $P = P_J \otimes P_{\comp{J}}$, 
      then, $\supp(P) = \supp(P_J) \times \supp(P_{\comp{J}})$.
      \item If $P=P_J\otimes P_\comp{J}$, $Q=Q_J\otimes Q_\comp{J}$,  $\supp(P_J)\subseteq\supp(Q_J)$, and $P(C)>0$, then $Q(C)>0$.
  \end{enumerate}
\end{lemma}
Note that $P$ and $Q$ do not have to factorize over $\Omg$ here.
\begin{proof}[Proof of \Cref{lemma:support_statements}]
  1. Let $\supp(P)\subseteq\supp(Q)$. We now prove the contraposition.
  If $Q(C) = 0$, then $C \cap \supp(Q) = \emptyset$ by definition of the support. It follows that $C \cap \supp(P) = \emptyset$, and thus, $P(C) = 0$.
  \\2. Let $P = P_J \otimes P_{\comp{J}}$.
  For all $\x \in \Omg_J$ and $\y \in \Omg_{\comp{J}}$, we have
  \begin{align*}
    \x \merge \y \in \supp(P) \quad & \Longleftrightarrow \quad 0 < P(\x \merge \y) = P_J(\x) P_{\comp{J}}(\y) \\
    & \Longleftrightarrow \quad
    P_J(\x) > 0 \text{ \ \ and \ \ } P_{\comp{J}}(\y) > 0 \\
    & \Longleftrightarrow \quad
    \x \in \supp(P_J) \text{ \ \ and \ \ } \y \in \supp(P_{\comp{J}}) \\
    & \Longleftrightarrow \quad
    \x \merge \y \in \supp(P_J) \times \supp(P_{\comp{J}}).
  \end{align*}
  3. Let $P=P_J\otimes P_\comp{J}$, and $Q=Q_J\otimes P_\comp{J}$, let $\supp(P_J)\subseteq\supp(Q_J)$, and let $P(C)>0$. By 2., we have
  \begin{equation}
  \supp(P)=\supp(P_J)\times\supp(P_\comp{J})\subseteq \supp(Q_J)\times \supp(P_\comp{J})=\supp(Q)\,.
  \end{equation}
  Since $P(C)>0$, it follows from 1. that $Q(C)>0$.
  This concludes the proof.
\end{proof}

We now use \Cref{lemma:support_statements} to show the following lemma.

\begin{lemma}
  \label{lemma:progressive_application_irrelevance}
  Let $A, C \subseteq \Omg$ be events and $P, Q \in \distributionsF_{C}(\Omg)$.
  Further, let $P_j = Q_j$ for all $j$ that are relevant to $A$ given $C$. 
  Then $P(A \given C) = Q(A \given C)$.
\end{lemma}

\begin{proof}
  Essentially, the proof is a progressive application of the definition of irrelevance. We change one factor at a time to transform $P$ into $Q$.
  However, as the fact that $j$ is irrelevant to $A$ given $C$ only makes a statement about distributions in $\distributionsF_C(\Omg)$, we need to prove that the intermediate distributions are all in $\distributionsF_C(\Omg)$. 

  Let $J \coloneq \historyC(A \given C)$.
  Let $Q$ be such that $Q_J$ is strictly positive. (Later we strengthen the claim to arbitrary distributions $Q\in\distributionsF_C(\Omg)$.)
  Let $P_j = Q_j$ for all $j$ that are relevant to $A$ given $C$.
  Since $J$ is finite, we can express it as $J = \{j_1, \dots, j_q\}$ with mutually distinct elements $j_k$.
  We now progressively replace factors in $P$ until it is equal to $Q$. We do this by defining a sequence $P^0, P^1, \dots , P^q$ of distributions, wherein $P^0=P$ and $P^q=Q$. Then we show that $P(A\given C)^{k-1}=P(A\given C)^k$ for all $k\in\{1, \dots, q\}$. We define $P^n$ recursively as follows.
  \begin{enumerate}
      \item For all $\iinI$, set $P^0_i\coloneq P_i$.
      \item For all $k \in \{1, \dots, q\}$ and $j \in I$, set $P^k_j \coloneq\begin{cases}P^{k-1}_j\text{ if $j \neq j_k$}\\  Q_j\text{ if $j = j_k$.}
      \end{cases}$
      \item Set $P^k = \bigotimes_{i \in I} P_i^k$.
  \end{enumerate}
Note that by construction, we have that $P^q = Q$, and also  $P^{k-1}$ and $P^{k}$ can only differ in $j_k$. Next we show by induction that for all $k\in \{1, \dots, q\}$, we have $P^k(C)>0$.
  First, we note that since $P^0 = P$, we have $P^0(C)>0$. We now assume that $P^{k-1}(C)>0$ (induction assumption).
  Since we assumed that $Q_J$ is strictly positive, $P^k_{j_k} = Q_{j_k}$ is strictly positive. As $P^{k}_{I \setminus \{j_k\}} = P^{k-1}_{I \setminus \{j_k\}}$, we have that $\supp(P^k) \supseteq \supp(P^{k-1})$ by Lemma~\ref{lemma:support_statements}. Therefore, $P^{k-1}(C) > 0$ implies $P^k(C) > 0$. That means, $P^k\in\distributionsF_C(\Omg)$ for all $k\in\{1,\dots,q\}$. As all $j_k$ are irrelevant to $A$ given $C$ by definition, we have that
  \begin{equation*}
    P(A \given C) = P^0(A \given C) = P^1(A \given C) = \dots = P^q(A \given C) = Q(A \given C)\,.
  \end{equation*}
  This proves the lemma for those $Q$ for which $Q_J$ is strictly positive. We now prove the general case. Let $Q\in \distributionsF_C(\Omg)$ be arbitrary.
  Let $Q'$ be a probability distribution with $Q'_{\comp{J}} = P_{\comp{J}} = Q_{\comp{J}}$ such that $Q'_{J}$ is strictly positive, for example $Q'_{J}$ could be chosen as the uniform distribution.  From \Cref{lemma:support_statements} part 2,
  it follows that $\supp(Q') \supseteq \supp(Q)$. Together with
  $Q(C) > 0$, this implies $Q'(C) > 0$ by \Cref{lemma:support_statements}, part 1.
  Then $P(A \given C) = Q'(A \given C)$. By the same argument with $Q$ instead of $P$, we obtain 
  $Q(A \given C) = Q'(A \given C)$, and so $P(A \given C) = Q(A \given C)$.
  This concludes the proof.
\end{proof}

Now we prove a simple preparatory lemma, namely that mixed independence satisfies the following decomposition graphoid property.

\begin{lemma}[Decomposition of mixed independence]
  \label{lemma:decomposition_graphoid_rule}
  Let $Y, Z$ be two random variables on $\Omg$ and let $B, C \subseteq \Omg$.
  Then,
  \begin{equation*}
    B \indep^P (Y, Z) \given C \quad \Longrightarrow \quad B \indep^P Y \given C.
  \end{equation*}
  \end{lemma}
  \begin{corollary}
  \label{cor:decomposition}
For $J' \subseteq J \subseteq I$, the independence $B \indep^P U_{J} \given C$ implies $B \indep^P U_{J'} \given C$.
  \end{corollary}

\begin{proof}
  Let $B\indep^P(Y,Z)\given C$.
  Let $y \in \Val(Y)$.
  Then we have
  \begin{align*}
    P(B \cap C) \cdot P\big(Y^{-1}(y) \cap C\big) &=
    \sum_{z \in \Val(Z)} P(B \cap C) \cdot P\big( (Y, Z)^{-1}(y, z) \cap C \big) \\
    &= \sum_{z \in \Val(Z)} P\big(B \cap (Y, Z)^{-1}(y, z) \cap C \big) \cdot P(C)&\text{\small(by $B\indep^P(Y,Z)\given C$)} \\
    &=  P\big( B \cap Y^{-1}(y) \cap C \big) \cdot P(C).
  \end{align*}
That shows $B \indep^P Y \given C$.
The corollary follows from $U_J = (U_{J'}, U_{J \setminus J'})$.
\end{proof}

\begin{lemma}
\label{lemma:cartesian_product_distribution}
    Let $J\subseteq I$, and let $P$ be a distribution over $\Omg$ with $P=P_J\otimes P_\comp{J}$. Further, let $A_J\subseteq \Omg_J$, and $A_\comp{J}\subseteq \Omg_\comp{J}$. Then we have $P(A_J\times A_\comp{J})=P_J(A_J)\cdot P_\comp{J}(A_\comp{J})$.
\end{lemma}

\begin{proof}
    \begin{align*}
    P_J(A_J)\cdot P_\comp{J}(A_\comp{J})&= \sum_{\omg_J\in A_J}P_J(\omg_J) \sum_{\omg_\comp{J}\in A_\comp{J}} P_\comp{J}(\omg_\comp{J})\\
    &= \sum_{\substack{\omg_J\in A_J\\  \omg_\comp{J}\in A_\comp{J}}}P_J(\omg_J)P_\comp{J}(\omg_\comp{J})\\
    &= \sum_{\omg_J\merge \omg_\comp{J}\in A_J\times A_\comp{J}}P(\omg_J\merge \omg_\comp{J}) &\text{\small($P=P_J\otimes P_\comp{J}$)}\\
    &=P(A_J\times A_\comp{J})
    \end{align*}
\end{proof}

For $J \subseteq I$, $\x \in \Omg_J$ and $D \subseteq \Omg$, we define $D^\alpha$ as 
\begin{equation}
\label{eq:D_alpha}
    D^\alpha\coloneq \{\omg\in D \colon \omg_J=\alpha\} =D\cap U^{-1}(\alpha)
\end{equation}
We obtain the following lemma.

\begin{lemma}
  \label{lemma:preparation_independence_cohistory}
  Let $P$ be a distribution that factorizes over $\Omg$, let $J \subseteq I$, $\x \in \Omg_J$, and $D \subseteq \Omg$. Further, let $\delta_\x \in \distributions(\Omg_J)$ be the delta distribution at $\x$, which is the distribution that assigns probability $1$ to $\x$.
  Then we have the following identities:
  \begin{enumerate}
    \item $P\big( D^\alpha
    ) = P_J(\x) \cdot P_{\comp{J}}\big(D^\alpha_\comp{J}\big)$. \\
    \item $\big( \delta_\x \otimes P_{\comp{J}} \big)(D) = P_{\comp{J}}\big(D^\alpha_\comp{J}\big) $.
  \end{enumerate}
\end{lemma}

\begin{proof}
  Part 1 follows from $D^\alpha = \{\x\} \times D^\alpha_\comp{J}$ and from \Cref{lemma:cartesian_product_distribution}. 
  For part $2$, we note
  \begin{align*}
    \big( \delta_\x \otimes P_{\comp{J}} \big)(D) &=\sum_{\omg\in D}(\delta_\alpha\otimes P_\comp{J})(\omg)=\sum_{\omg\in D}\delta_\alpha(\omg_J)P_\comp{J}(\omg_\comp{J})=\sum_{\omg\in D^\alpha}P_\comp{J}(\omg_\comp{J})=\sum_{\omg_\comp{J}\in D^\alpha_\comp{J}}P_\comp{J}(\omg_\comp{J})
    = P_{\comp{J}}\big( D^\alpha_\comp{J} \big)\,.
  \end{align*}
\end{proof}

Ultimately, the goal is to show that $\comp{\history{A \given C}} = \historyC(A \given C)$ (\Cref{lemma:cohistory_is_complement}).
First, we show that $U_\comp{\history{A \given C}}$ and $U_{\historyC(A \given C)}$ have similar independence properties, namely that $U_{J} \indepstar U_{\comp{J}} \given C$ holds both for $J=\history(A\given C)$ (\Cref{lemma:similar_indep_properties}) and for $J=\historyC(A\given C)$ (\Cref{lemma:cohistory_has_independence_property}). Note that \Cref{lemma:similar_indep_properties} is not necessary for our proof of \Cref{lemma:cohistory_is_complement}, but rather serves as a motivation why \Cref{lemma:cohistory_has_independence_property} is a useful step in the direction of showing that $\comp{\history{A \given C}} = \historyC(A \given C)$.
\begin{lemma}
    \label{lemma:similar_indep_properties}
Let $J = \history(A \given C)$.
Then we have $U_{J} \indepstar U_{\comp{J}} \given C$.
\end{lemma}
\begin{proof}
    Let $P \in \distributionsF(\Omg)$, $\x \in \Omg_J$ and $\y \in \Omg_{\comp{J}}$. We first observe that $\{U_J = \alpha\}\cap (C_J\times C_\comp{J})=(\{\alpha\}\cap C_J)\times C_\comp{J}$. Since $C=C_J\times C_\comp{J}$ by definition of the history, we have
\begin{align*}
  P\big( U_J^{-1}(\x) \cap C \big)  P\big( U_{\comp{J}}^{-1}(\y) \cap C \big) &=
  P\Big( \big(\{\x\} \cap C_J\big) \times C_{\comp{J}} \Big)
  P\Big( \big( \{\y\} \cap C_{\comp{J}} \big) \times C_{J} \Big) \\
  &= P_J\big( \{\x\} \cap C_J \big)  P_{\comp{J}}(C_{\comp{J}})P_{\comp{J}}\big( \{\y\} \cap C_{\comp{J}} \big)  P_J(C_J) &\text{\small(\Cref{lemma:cartesian_product_distribution})}\\
  &= P\big( \{\x \cdot \y\} \cap C \big)  P(C) &\text{\small(\Cref{lemma:cartesian_product_distribution})}\\
  &= P\big( U_J^{-1}(\x) \cap U_{\comp{J}}^{-1}(\y) \cap C \big)  P(C).
\end{align*}
This precisely means that $U_{J} \indepstar U_{\comp{J}} \given C$.
\end{proof}
In~\Cref{lemma:cohistory_has_independence_property}, we show that this independence also holds with $\historyC(A \given C)$ in place of $\history{A \given C}$.
We first show that the independence holds for A and $U_{\historyC(A \given C)}$.

\begin{lemma}
  \label{lemma:intermediate_independence_claim}
  Let $A, C \subseteq \Omg$, and let $J = \historyC(A \given C)$.
  Then $A \indepstar U_J \given C$.
\end{lemma}

\begin{proof}
  Let $P \in \distributionsF(\Omg)$ and $\x \in \Omg_J$.
  Let $B \coloneq U_J^{-1}(\x)$.
  We need to show $A \indep^P B \given C$, i.e., $P(A \cap C) \cdot P(B \cap C) = P(A \cap B \cap C) \cdot P(C)$.
  Note that if $P(B \cap C) = 0$, then both sides of the equation are zero and the independence holds. So we can assume that $P(B \cap C) > 0$.

  Let $\delta_\x \in \distributionsF(\Omg_J)$ be the delta distribution at $\x$.
  For this proof, we would like to apply \Cref{lemma:progressive_application_irrelevance} to the distributions $(\delta_\alpha\otimes P_\comp{J})$, and $P$. The conditions we need in order to apply \Cref{lemma:progressive_application_irrelevance} are
  \begin{enumerate}
      \item $P\in\distributionsF_C(\Omg)$
      \item $(\delta_\alpha\otimes P_\comp{J})\in \distributionsF_C(\Omg)$
      \item $(\delta_\x \otimes P_{\comp{J}})_i=P_i$ for all factors $i$ that are relevant to $A$ given $C$.
  \end{enumerate}
1. is satisfied since $P\in\distributionsF(\Omg)$ by our assumption, and $P(C)>0$ because $P(B\cap C)>0$.
For 2., we need to show that $(\delta_\alpha\otimes P_\comp{J})\in \distributionsF(\Omg)$ and  $(\delta_\alpha\otimes P_\comp{J})(C)>0$.
  First, we note that $\delta_\alpha = \bigotimes_{i\in J}\delta_{\alpha_i}$ since for any $\alpha'\in \Omg_J$ we have
  \begin{equation}
  \label{eq:delta_factorizes}
      \prod_{i\in J}\delta_{\alpha_i}(\alpha'_i)=1 \iff \alpha'_i=\alpha_i \text{ for all $i\in J$} \iff  \alpha'=\alpha\iff \delta_\alpha(\alpha') = 1\,.
  \end{equation}
  Since $P$ factorizes over $\Omg$, $P_\comp{J}$ also factorizes, and thus $(\delta_\alpha\otimes P_\comp{J}) \in \distributionsF(\Omg)$.

  Next, we show that $(\delta_\alpha\otimes P_\comp{J})(C)>0$.
Since $B=U_J^{-1}(\alpha)$, $B\cap C =
\{\omg\in C\colon \omg_J=\alpha\}=C^\alpha$. By \Cref{lemma:preparation_independence_cohistory} part 1, it follows that
\begin{equation}
    \label{eq:BcapC_probability}
P(B\cap C)=P(C^\alpha)=P_J(\alpha)\cdot P_\comp{J}(C^\alpha_\comp{J})\,.
\end{equation}
Since $P(B\cap C)>0$, we also have $P_\comp{J}(C^\alpha_\comp{J})>0$.
  Further, by \Cref{lemma:preparation_independence_cohistory} part 2., we have that
    $\big( \delta_\x \otimes P_{\comp{J}} \big)(C) = P_{\comp{J}}\big( C^\alpha_\comp{J} \big)
    > 0$.
  Therefore, we have $(\delta_\alpha \otimes P_\comp{J})\in \distributionsF_C(\Omg)$, and the second condition for \Cref{lemma:progressive_application_irrelevance} is satisfied.

  For proving 3., we observe that the set of factors that are relevant to $A$ given $C$ is exactly $\comp{J}$, since $J$ is the cohistory of $A$ given $C$. Further, since $(\delta_\x \otimes P_{\comp{J}})=\delta_\x \otimes (\bigotimes_{i\in \comp{J}}\ P_i)$, we have $(\delta_\x \otimes P_{\comp{J}})_i=P_i$ for all $i \in \comp{J}$, and 3. is satisfied.

  We can now apply    \Cref{lemma:progressive_application_irrelevance}, and obtain $P(A \given C) = (\delta_\x \otimes P_{\comp{J}})(A \given C)$.
  The right-hand side can be further transformed as follows (using Lemma~\ref{lemma:preparation_independence_cohistory}, part 2 in step $(*)$):
  \begin{align*}
    (\delta_\x \otimes P_{\comp{J}})(A \given C) = \frac{(\delta_\x \otimes P_{\comp{J}})(A \cap C)}{(\delta_\x \otimes P_{\comp{J}})(C)}
    \overset{(*)}{=} \frac{P_{\comp{J}}\big((A \cap C)^\alpha_\comp{J}\big)}{P_{\comp{J}}\big(C^\alpha_\comp{J}\big)} = \frac{P_{\comp{J}}\big( A^\alpha_\comp{J} \cap C^\alpha_\comp{J} \big)}{P_{\comp{J}}\big(C^\alpha_\comp{J}\big)} = P_{\comp{J}}\big( A^\alpha_\comp{J} \given C^\alpha_\comp{J} \big).
  \end{align*}
  So overall, we have
  \begin{equation}
      \label{eq:P_A_given_C_equal}
  P(A \given C) = P_{\comp{J}}\big(A^\alpha_\comp{J} \given C^\alpha_\comp{J}\big)\,.
  \end{equation}
  Combining these results, we obtain
  \begin{align*}
    P(A \given C) \cdot P(B \cap C) &= P_{\comp{J}}\big( A^\alpha_\comp{J} \given C^\alpha_\comp{J} \big) \cdot P_J(\x) \cdot P_{\comp{J}}\big( C^\alpha_\comp{J} \big) &\text{\small\eqref{eq:BcapC_probability} and \eqref{eq:P_A_given_C_equal}}\\
    &= P_J(\x) \cdot P_{\comp{J}}\big(A^\alpha_\comp{J} \cap C^\alpha_\comp{J}\big) \\
    &= P_J(\x) \cdot P_{\comp{J}}\big((A \cap C)^\alpha_\comp{J}  \big)&\text{\small(distributivity of projection and \eqref{eq:D_alpha})}\\
    &= P((A\cap C)^\alpha) &\text{\small(\Cref{lemma:preparation_independence_cohistory} part 1)}\\
    &=P(A\cap C \cap U_J^{-1}(\alpha))&\text{\small \eqref{eq:D_alpha}}\\
    &= P(A \cap B \cap C).&\text{\small(definition of $B$)}
  \end{align*}
  Multiplying both sides with $P(C)$, it follows that $A\indep^P B\given C$. This concludes the proof.
\end{proof}

\begin{lemma}
  \label{lemma:cohistory_has_independence_property}
  Let $A, C \subseteq \Omg$ and $J = \historyC(A \given C)$ be the cohistory of $A$ given $C$.
  Then $U_J \indepstar U_{\comp{J}} \given C$.
\end{lemma}

\begin{proof}
  Let $\x \in \Omg_{J}$ and let $B \coloneq U_J^{-1}(\x)$.
  We need to show that $B \indepstar U_{\comp{J}} \given C$.
  Recall that by \Cref{lemma:intermediate_independence_claim}, we have
  $A \indepstar U_{\historyC(A\given C)}$
  $ \given C$, so we have $A \indepstar B \given C$. By \Cref{lemma:completeness_cohistory}, it follows that $\historyC(A \given C) \cup \historyC(B \given C) = I$, and therefore $\comp{J} \subseteq \historyC(B \given C)$. Moreover,
  \Cref{lemma:intermediate_independence_claim} also implies that
  $B \indepstar U_{\historyC(B \given C)} \given C$.
  Since $\comp{J}\subseteq \historyC(B\given C) $, \Cref{cor:decomposition} implies that
  $B \indepstar U_{\comp{J}} \given C$.
  This concludes the proof.
\end{proof}

\begin{lemma}
  \label{lemma:cohistory_disintegrates}
  Let $A, C \subseteq \Omg$ and $J = \historyC(A \given C)$.
  Then $J$ disintegrates $C$, i.e., $C = C_J \times C_{\comp{J}}$.
\end{lemma}

\begin{proof}
  Let $P \in \distributionsF(\Omg)$ be any strictly positive distribution, for example the uniform distribution on each factor.
  Define $Q(D) \coloneqq P(D \given C)$ for any event $D \in \Omg$.
  Then $Q$ is a distribution on $\Omg$ that may \emph{not} necessarily factorize over all $i \in I$.
  However, we show that it factors into two components $Q_J$ and $Q_\comp{J}$ with $\supp(Q_J)=C_J$ and $\supp(Q_\comp{J})=C_\comp{J}$.

Let $\alpha\in\Omg_J, \beta\in\Omg_\comp{J}$ as before. Further, let $Q_J$ and $Q_\comp{J}$ be the marginal distributions of $Q$ as in \Cref{def:marginal}, which means that we have $Q_J(\x) = Q\big(U_J^{-1}(\x)\big)$ and $Q_\comp{J}(\x) = Q\big(U_\comp{J}^{-1}(\y)\big)$.
  Then $Q$ is the outer product $Q = Q_J \otimes Q_{\comp{J}}$, because 
  \begin{align*}
    Q(\x \cdot \y) &= P(\x \cdot \y \given C) &\text{\small(definition of $Q$)} \\
    &= P\big( U_J^{-1}(\x) \cap U_{\comp{J}}^{-1}(\y) \given C \big) \\
    &
    =P\big( U_J^{-1}(\x) \given C \big) \cdot P\big( U_{\comp{J}}^{-1}(\y) \given C \big) &\text{\small(\Cref{lemma:cohistory_has_independence_property})}\\
    &= Q_J(\x) \cdot Q_{\comp{J}}(\y).
  \end{align*}
  Since $P$ is strictly positive and $Q = P(\cdot \given C)$, we have $\supp(Q) = C$.
We also have $\supp(Q_J) = C_J$, and analogously $\supp(Q_{\comp{J}}) = C_{\comp{J}}$, by the following derivation.
  \begin{align*}
    \x \in \supp(Q_J) \quad &\Longleftrightarrow \quad 0 < Q_J(\x) = Q\big( U_J^{-1}(\x) \big)  \\
    &\Longleftrightarrow \quad 0< P\big( U_J^{-1}(\x) \given C \big) = \frac{P\big( U_J^{-1}(\x) \cap C \big)}{P(C)} \\
    &\Longleftrightarrow \quad P\big( U_J^{-1}(\x) \cap C \big)>0 \\
    & \Longleftrightarrow \quad \nothing \neq U_J^{-1}(\x) \cap C &\text{\small($P$ strictly positive)}\\
    & \Longleftrightarrow \quad \x \in C_J.
  \end{align*}
  Since $Q=Q_J\otimes Q_\comp{J}$, \Cref{lemma:support_statements} implies that $\supp(Q) = \supp(Q_J) \times \supp(Q_{\comp{J}})$
  .   Since $\supp(Q)=C$, $\supp(Q_J)=C_J$, and $\supp(Q_\comp{J})=C_\comp{J}$, we have $C=C_J\times C_\comp{J}$, which concludes the proof.
\end{proof}

We are now ready to prove \Cref{lemma:cohistory_is_complement}, which states that $\comp{\history(A\given C)}=\historyC(A\given C)$.
\begin{proof}[Proof of~\Cref{lemma:cohistory_is_complement}]
  \label{proof:cohistory_is_complement}
  Let $J \coloneqq \history{A \given C}$.
  We need to show $\comp{J} = \historyC(A \given C)$.
  For the inclusion $\comp{J} \subseteq \historyC(A \given C)$, let $i \in \comp{J}$.
  Let $(P, Q) \in \distributionsF_{C, i}(\Omg)$.
  Since $J$ generates $A$ given $C$ by \Cref{lemma:historygenerates}, \Cref{lemma:characterization_event_generation} implies that $A \cap C = (A \cap C)_J \times C_{\comp{J}}$. Further, by definition of the history, we have $C = C_J \times C_{\comp{J}}$.
  We obtain:
  \begin{align*}
    P(A \given C) = \frac{P(A \cap C)}{P(C)}
    = \frac{P\big( (A \cap C)_J \times C_{\comp{J}} \big)}{P\big(C_J \times C_{\comp{J}}\big)}
    \overset{(*)}{=} \frac{P_J\big( (A \cap C)_J \big)  P_{\comp{J}}(C_{\comp{J}})}{P_J(C_J)  P_{\comp{J}}(C_{\comp{J}})} = \frac{P_J\big( (A \cap C)_J \big)}{P_J(C_J)}.
  \end{align*}
  In the step $(*)$, we use \Cref{lemma:cartesian_product_distribution} and the fact that $P$ factorizes over $\Omg$.
  Similarly, we get
  \begin{equation*}
    Q(A \given C) = \frac{Q_J\big( (A \cap C)_J \big)}{Q_J(C_J)}.
  \end{equation*}
  Since $(P, Q) \in \distributionsF_{C, i}(\Omg)$, and $i \notin J$ by our assumption, we have $P_J = Q_J$, and thus $P(A \given C) = Q(A \given C)$.
  Thus, $i \in \historyC(A \given C)$.

  Finally, we show the other inclusion: $\history{A \given C} \subseteq \comp{\historyC(A \given C)} \eqqcolon K$. 
By~\Cref{lemma:characterization_event_generation}, the history $\history(A \given C)$ is the smallest set $J$ which satisfies $A \cap C = (A \cap C)_J \times C_{\comp{J}}$ and $C=C_J\times C_\comp{J}$.
  Thus, it suffices to show that $K$ has these properties.
  By~\Cref{lemma:cohistory_disintegrates}, we know that $C=C_K\times C_\comp{K}$.
  Thus, what remains is to show that $A \cap C = (A \cap C)_K \times C_{\comp{K}}$.

  For the inclusion from left to right,
  let $\omg\in A\cap C$. It follows that $\omg_K\in (A\cap C)_K$. Since $\omg \in C$, we have $\omg_\comp{K}\in C_\comp{K}$, and thus $\omg =\omg_K\merge \omg_\comp{K} \in (A \cap C)_K \times C_{\comp{K}}$.

  For the inclusion from right to left,
   let $\omg \in A \cap C$ and $\alpha \in C_{\comp{K}}$.
  It suffices to show that $\omg_K \merge \alpha \in A \cap C$, or equivalently $\delta_{\omg_K \merge \alpha}(A \cap C) = 1$. To show this, we would like to apply \Cref{lemma:progressive_application_irrelevance} to $\delta_{\omg}$ and $\delta_{\omg_K \cdot \alpha}$, so we need to show that both are in $\distributionsF_C(\Omg)$.

  By \eqref{eq:delta_factorizes}, we have $\delta_{\omg}, \delta_{\omg_K \cdot \alpha} \in \distributionsF(\Omg)$. Next we show that both assign positive probability to $C$.
  Since $\omg\in A\cap C$, we have $\omg\in C$ and thus $\delta_\omg(C)=1>0$.
  Since $\omg_K\in (A\cap C)_K$ so $\omg_K\in C_K$, we have $\delta_{\omg_K}(C_K)=1$. Since $\alpha\in C_\comp{K}$, we have $ \delta_\alpha(C_{\comp{K}})=1$. It follows that
  \begin{equation*}
    \delta_{\omg_K \cdot \alpha}(C) = \delta_{\omg_K \cdot \alpha}\big( C_K \times C_{\comp{K}} \big) = \delta_{\omg_K}(C_K) \cdot \delta_\alpha(C_{\comp{K}}) = 1 \cdot 1 = 1 > 0.
  \end{equation*}
  Thus, $\delta_{\omg}, \delta_{\omg_K \cdot \alpha}$ are two distributions in $\distributionsF_{C}(\Omg)$ with the same factors in $K = \comp{\historyC(A \given C)}$.
  From~\Cref{lemma:progressive_application_irrelevance}, it follows that $\delta_\omg(A \given C) = \delta_{\omg_K \cdot \alpha}(A \given C)$, and therefore
  \begin{align*}
    \delta_{\omg_K \cdot \alpha}(A \cap C)= \frac{\delta_{\omg_K \cdot \alpha}(A \cap C)}{\delta_{\omg_K \cdot \alpha}(C)} = \delta_{\omg_K \cdot \alpha}(A \given C) =\delta_\omg(A \given C)=\frac{\delta_\omg(A \cap C)}{\delta_\omg(C)}  = 1.\\
  \end{align*}
  Therefore, $\omg_K \merge \alpha \in A \cap C$. This shows that $A\cap C=(A\cap C)_K\times C_\comp{K}$. Therefore, $\history(A\given C)\subseteq K$.
  This concludes the proof of~\Cref{lemma:cohistory_is_complement}, and thus of the completeness of structural independence for events,~\Cref{lemma:completeevent}, and consequently the main theorem,~\Cref{thrm:soundcomplete}.
\end{proof}

\end{document}